\documentclass[10pt]{article}

\usepackage{fullpage}
\usepackage{blindtext}
\usepackage{booktabs} 
\usepackage{makecell} 
\usepackage[round]{natbib}
\usepackage{url}
\usepackage{wrapfig} 
\usepackage{caption}

\usepackage{amssymb}
\usepackage{mathtools}
\usepackage{amsthm}
\usepackage{nameref}
\usepackage{hyperref}
\usepackage[capitalize,noabbrev]{cleveref}

\usepackage{amsmath}
\usepackage{amsfonts}
\usepackage{graphicx}
\usepackage[right=1in,left=1in,top=1in,bottom=1in]{geometry}
\usepackage[svgnames,dvipsnames]{xcolor}
\usepackage{mathrsfs,bm}
\usepackage{amssymb}
\usepackage{titlesec}
\usepackage{tikz-cd}
\usepackage[most]{tcolorbox}
\usepackage{etoolbox}
\usepackage{ stmaryrd }
\usepackage[T1]{fontenc}
\usepackage[shortlabels]{enumitem}

\usepackage[english]{babel}
\usepackage[utf8]{inputenc}
\usepackage{fancyhdr}
\usepackage{mathtools}
\usepackage{microtype}
\usepackage[normalem]{ulem}
\usepackage{stmaryrd}
\usepackage{wasysym}
\usepackage{multirow}
\usepackage{prerex}

\theoremstyle{plain}

\titleformat{\section}[block]{\color{black}\Large\bfseries}{\thesection}{1em}{}
\titleformat{\subsection}[hang]{\color{black}\large\bfseries}{\thesubsection}{1em}{}

\titleformat{\subsubsection}[hang]{\color{black}\large\bfseries}{\thesubsubsection}{1em}{}

\theoremstyle{plain}
\newtheorem{thm}{Theorem}

\newtheorem{lemma}[thm]{Lemma}

\newtheorem{cor}[thm]{Corollary}

\newtheorem{assumption}{Assumption}

\newtheorem{rmk}{Remark}

\newtheorem*{assumption*}{Assumption}

\newtheorem{defn}[thm]{Definition}
\newtheorem{fact}[thm]{Fact}

\newtheorem{goal}{Goal}

\newtheorem*{conj*}{Conjecture}
\newtheorem*{defn*}{Definition}
\newtheorem*{note*}{Notation}
\newtheorem*{fact*}{Fact}
\newtheorem*{ques*}{Question}
\newtheorem*{exer*}{Exercise}
\newtheorem*{prob*}{Problem}

\newtheorem*{algo*}{Algorithm}

\usepackage{cleveref}
\usepackage[most]{tcolorbox}
\newtcbtheorem{solb}{Solution}{
                lower separated=false,
                colback=white!80!gray,
colframe=white, fonttitle=\bfseries,
colbacktitle=white!50!gray,
coltitle=black,
enhanced,
breakable,
boxed title style={colframe=black},
attach boxed title to top left={xshift=0.5cm,yshift=-2mm},
}{solb}

\Crefname{defn}{Definition}{Definitions}
\Crefname{definition}{Definition}{Definitions}
\Crefname{rmk}{Remark}{Remarks}
\Crefname{prop}{Proposition}{Propositions}
\Crefname{thm}{Theorem}{Theorems}
\Crefname{theorem}{Theorem}{Theorems}
\Crefname{cor}{Corollary}{Corollaries}
\Crefname{lemma}{Lemma}{Lemmas}
\Crefname{algo}{Algorithm}{Algorithms}
\Crefname{ex}{Example}{Examples}
\Crefname{answer}{Answer}{Answers}
\Crefname{ques}{Question}{Questions}
\Crefname{prob}{Problem}{Problems}
\Crefname{assumption}{Assumption}{Assumptions}
\Crefname{note}{Notation}{Notations}
\Crefname{fact}{Fact}{Facts}
\Crefname{exer}{Exercise}{Exercises}
\Crefname{conj}{Conjecture}{Conjectures}
\Crefname{claim}{Claim}{Claims}
\Crefname{figure}{Figure}{Figures}
\Crefname{subsection}{Subsection}{Subsections}
\Crefname{section}{Section}{Sections}
\Crefname{appendix}{Appendix}{Appendices}
\Crefname{table}{Table}{Tables}
\crefformat{equation}{(#2#1#3)}
\Crefname{algocf}{Algorithm}{Algorithms}
\crefname{algocfproc}{Algorithm}{Algorithms}

\usepackage[algo2e, vlined, ruled, linesnumbered]{algorithm2e}

\SetCommentSty{mycommfont}

\usepackage{etoolbox}  
\makeatletter
\patchcmd{\algorithmic}{\addtolength{\ALC@tlm}{\leftmargin} }{\addtolength{\ALC@tlm}{\leftmargin}}{}{}
\makeatother

\newcommand{\nonl}{\renewcommand{\nl}{\let\nl}}

\SetKwRepeat{Do}{do}{while}

\newcommand\numberthis{\addtocounter{equation}{1}\tag{\theequation}}

\crefalias{AlgoLine}{line}%
\makeatletter
\let\cref@old@stepcounter\stepcounter
\def\stepcounter#1{%
  \cref@old@stepcounter{#1}%
  \cref@constructprefix{#1}{\cref@result}%
  \@ifundefined{cref@#1@alias}%
    {\def\@tempa{#1}}%
    {\def\@tempa{\csname cref@#1@alias\endcsname}}%
  \protected@edef\cref@currentlabel{%
    [\@tempa][\arabic{#1}][\cref@result]%
    \csname p@#1\endcsname\csname the#1\endcsname}}
\makeatother

\usepackage{mathtools}
\makeatletter
\newcommand{\mytag}[2]{%
  \text{#1}%
  \@bsphack
  \begingroup
    \@onelevel@sanitize\@currentlabelname
    \edef\@currentlabelname{%
      \expandafter\strip@period\@currentlabelname\relax.\relax\@@@%
    }%
    \protected@write\@auxout{}{%
      \string\newlabel{#2}{%
        {#1}%
        {\thepage}%
        {\@currentlabelname}%
        {\@currentHref}{}%
      }%
    }%
  \endgroup
  \@esphack
}
\makeatother

\definecolor{aqua}{rgb}{0.0, 1.0, 1.0}
\definecolor{caribbeangreen}{rgb}{0.0, 0.8, 0.6}
\definecolor{azure}{rgb}{0.0, 0.5, 1.0}
\definecolor{charcoal}{rgb}{0.21, 0.27, 0.31}

\hypersetup{
    colorlinks=true,
    breaklinks=true,
	citecolor=Mulberry,
	filecolor=black,
	linkcolor=azure,
	urlcolor=PineGreen,
	linkbordercolor=blue
}

\usepackage{subfigure}
\usepackage[subfigure]{tocloft}
\makeatletter
\def\clearwf{\par{\count@\c@WF@wrappedlines\zz}\par}

\def\zz{{%
\ifnum\count@>\@ne
\noindent\mbox{zz}\\%
\advance\count@\m@ne
\expandafter\zz
\else
\ifhmode\unskip\unpenalty\fi
\fi}}

\makeatother

\usepackage{joe/joe_macros}
\usepackage{joe/bandit_macros}
\usepackage{floatrow}
\usepackage{bookmark}
\usepackage{cancel}

\newfloatcommand{capbtabbox}{table}[][\FBwidth]

\title{
Tracking Most Significant Shifts in Infinite-Armed Bandits
}
\author{
Joe Suk \\
Columbia University\\
\href{mailto:joe.suk@columbia.edu}{{ \texttt{joe.suk@columbia.edu}}}%
\and
Jung-hun Kim\\
Seoul National University\\
\href{mailto:junghunkim@snu.ac.kr}{{ \texttt{junghunkim@snu.ac.kr}}}%
}
\date{}

\begin{document}
\maketitle

\begin{abstract}
We study an infinite-armed bandit problem where actions' mean rewards are initially sampled from a {\em reservoir distribution}.
Most prior works in this setting focused on stationary rewards \citep{berry97,wang08,bonald13,carpentier15} with the more challenging adversarial/non-stationary variant only recently studied in the context of rotting/decreasing rewards \citep{kim22,kim24}.
Furthermore, optimal regret upper bounds were only achieved using parameter knowledge of non-stationarity and only known for certain regimes of regularity of the reservoir.
This work shows the first parameter-free optimal regret bounds while also relaxing these distributional assumptions.

We first introduce a blackbox scheme to convert a finite-armed MAB algorithm designed for near-stationary environments into a parameter-free algorithm for the infinite-armed non-stationary problem with optimal regret guarantees.

We next study a natural notion of {\em significant shift} for this problem inspired by recent developments in finite-armed MAB \citep{suk22}.
We show that tighter regret bounds in terms of significant shifts can be adaptively attained by employing a randomized variant of elimination within our blackbox scheme.
Our enhanced rates only depend on the rotting non-stationarity and thus exhibit an interesting phenomenon for this problem where rising non-stationarity does not factor into the difficulty of non-stationarity.
\end{abstract}

\section{Introduction}

We study the multi-armed bandit (MAB) problem, where an agent sequentially plays arms from a set $\mc{A}$, based on partial and random feedback for previously played arms called rewards.
The agent's goal is to maximize earned rewards.

Much of the classical literature \citep[see][for surveys]{bubeck2012a,slivkinsbook,lattimore} focuses on finite armed bandits where $\mc{A} = [K]$ for some fixed $K \in \mb{N}$.
The theory here then typically assumes a large time horizon of play $T$ relative to $K$.
However, in practice, the number of arms can be prohibitively large as is the case in recommendation engines or adaptive drug design motivating the so-called {\em many-armed}, or {\em infinite-armed}, model.

At the same time, another practical reality is that of changing reward distributions or {\em non-stationarity}.
While there has been a surge of works here \citep{kocsis2006,yu2009,garivier2011,mellor13,liu2018,auer2019,chen2019,cao2019,manegueu2021,wei2021,suk22,jia2023,abbasi22,suk24}, most works here again focus on the finite-armed problem.

This work studies infinite-armed bandits where mean rewards of actions are initially drawn from a {\em reservoir distribution} and evolve over time under rested non-stationarity.  While much of the existing literature on this topic focuses on stationary rewards \citep{berry97,wang08,bonald13,carpentier15}, the more challenging adversarial or non-stationary scenario has only recently been explored in the context of rotting (i.e., decreasing) rewards \citep{kim22,kim24}.
\citet{kim24}'s state-of-the-art algorithm for rotting bandits relies on prior knowledge of non-stationarity parameters and further regularity assumptions on the reservoir distribution to attain optimal regret bounds.
However, these assumptions can be impractical in real-world applications.

This work studies a broader non-stationary model where rewards are decided by an adaptive adversary and aims to derive regret bounds without requiring algorithmic knowledge of non-stationarity.
Additionally, we go beyond the task of attaining optimal regret bounds as posed by \citet{kim24}, and show enhanced regret bounds which can be tighter (i.e., possibly near-stationary rates) despite large non-stationarity, thus more properly capturing the theoretical limits of learnability in changing environments.
Such insights are inspired by similar developments in the finite-armed analogue \citep{suk22}, where substantial changes in best arm, called {\em significant shifts}, can be detected leading to more conservative procedures which don't overestimate the severity of non-stationarity.

\subsection{More on Related Works}

The most relevant works are \citet{kim22,kim24}.
The first of these works studies a simpler model where the reservoir distribution is uniform on $[0,1]$ and there's a fixed upper bound $\rho$ on the magnitude of round-to-round non-stationarity.
\citet{kim22} show the minimax regret rate is $\rho^{1/3}\cdot T + \sqrt{T}$ and derive a matching upper bound, up to $\log$ terms, but using algorithmic knowledge of $\rho$ as well as a suboptimal regret upper bound without knowledge of $\rho$.

\citet{kim24} study a general setting where the reservoir distribution for initial mean reward $\mu_0(a)$ of arm $a$ satisfies $\mb{P}(\mu_0(a) > 1-x) = \Theta(x^{\beta})$ for all $x\in (0,1)$.
In this case, the minimax regret rate is order $\min\{ V^{\frac{1}{\beta+2}}\cdot T^{\frac{\beta+1}{\beta+2}}, L^{\frac{1}{\beta+1}}\cdot T^{\frac{\beta}{\beta+1}}\}$ in terms of number of changes $L$ in rewards or total variation $V$ (quantifying magnitude of changes over time).
With knowledge of $L$ and $V$, they show a matching regret upper bound for $\beta \geq 1$ and suboptimal bounds for $\beta<1$ or without knowledge of $L$ and $V$.

We also note that non-stationarity is well-studied in more structured many or infinite-armed bandit settings such as (generalized) linear \citep{cheung2019learning,russac2020algorithms,zhao2020,kim20,wei2021}, kernel \citep{hong23,iwazaki24}, or convex bandits \citep{wang22}.
To contrast, our infinite-armed setting does not assume any metric structure on the arm space and so these works are not easily to comparable to this paper.
We also note there are no known results on more nuanced measures of non-stationarity, like the significant shifts, even for such structured settings.


\subsection{Contributions}

\begin{itemize}
	\item We show the first optimal and adaptive (a.k.a. parameter-free) regret upper bounds for non-stationary infinite-armed bandits (\Cref{cor:elim-bounds}). 
	In fact, our bounds are expressed in terms of tighter and more optimistic measures of non-stationarity (\Cref{subsec:nonstat} and \Cref{defn:sig-shift}) new to this work.
	This resolves open questions of \citet{kim22,kim24}.

	We note our procedures are substantially different from \citet{kim22,kim24} who rely on explore-then-exploit strategies whereas we revisit the subsampling approach of \citet{bayati20} which partially reduces the problem to analyzing finite-armed bandits.
	\item Along the way, we develop the first high-probability regret bounds for the infinite-armed setting. 
	Our regret upper bound in \Cref{cor:elim-bounds} also relaxes distributional assumptions on the reservoir distribution, not requiring an upper bound on the masses of randomly sampled rewards.
	Both such generalizations were unknown in prior works even in the stationary setting.
	\item To our knowledge, our work is the first to develop adaptive dynamic regret bounds of the style $V^{1/3} T^{2/3} \land \sqrt{LT}$ with bandit feedback and adaptively adversarial changes.
	Notably, such results are yet unknown even in the finite-armed setting.
	\item We validate our findings via experiments on synthetic data, showing our procedures perform better than the state-of-art algorithms for rotting infinite-armed bandits.
\end{itemize}

\section{Setup}

\subsection{Non-Stationary Infinite-Armed Bandits}

We consider a multi-armed bandit with infinite armset $\mc{A}$.
At each round $t$, the agent plays an arm $a_t$, choosing either to newly sample $a_t$ from $\mc{A}$ or to play an already sampled arm among the previously chosen arms $\{a_1,\ldots,a_{t-1}\}$.

When the agent samples arm $a_t \in \mc{A}$ at round $t$, it observes a random reward $Y_t(a_t) \in [0,1]$ with mean $\mu_t(a_t) \in [0,1]$ whose value is randomly drawn from a {\em mean reservoir distribution}.
The reward of this chosen arm in the subsequent round is then decided according to an adaptive adversary with access to prior decisions $\{a_s\}_{s\leq t}$ and observations $\{ Y_t(a_s)\}_{s\leq t}$.
As in prior works \citep{berry97,wang08,carpentier15,bayati20,kim24}, we assume the reservoir distribution is $\beta$-regular, parametrized by a shape parameter $\beta > 0$.

\begin{assumption}[$\beta$-Regular Reservoir Distribution]\label{ass:reservoir}
	We assume a {\bf $\beta$-regular} reservoir distribution for some $\beta > 0$: there exists constants $\kappa_1,\kappa_2 > 0$ such that for all $x \in [0,1]$:
\[
	\kappa_1 \cdot x^{\beta} \leq \mb{P}(\mu_0(a) > 1 - x) \leq \kappa_2 \cdot x^{\beta}.
\]
\end{assumption}

\begin{rmk}
	Prior works on non-stationary bandits \citep{besbes2014,kim22,kim24} allow for unplayed arms' rewards to change each round, which is equivalent to our setting with a different accounting of changes, as changes in yet unplayed arms do not affect performance.
\end{rmk}

Let $\delta_t(a) := 1 - \mu_t(a)$ be the gap of arm $a$ at round $t$.
Then, the {\em cumulative regret} is $\bld{R}_T := \sum_{t=1}^T \delta_t(a_t)$.
Let $\delta_t(a,a') := \mu_t(a) - \mu_t(a')$ be the {\em relative regret} of arm $a'$ to $a$ at round $t$.

\subsection{Non-Stationarity Measures}\label{subsec:nonstat}

Let $V := \sum_{t=2}^T |\mu_{t-1}(a_{t-1}) - \mu_t(a_{t-1})|$ denote the {\em realized total variation} which measures the total variation of changes in mean rewards through the sequence of played arms.
In \Cref{sec:sig-shifts}, we also consider the {\em total realized rotting variation} $V_R := \sum_{t=2}^T (\mu_{t-1}(a_{t-1}) - \mu_t(a_{t-1}))_+$ which sums the magnitudes of rotting in rewards over time.

Let $L := \sum_{t=2}^T \pmb{1}\{ \mu_t(a_{t-1}) \neq \mu_{t-1}(a_{t-1})\}$ denote the {\em realized count of changes}.
We also consider the {\em realized number of rotting changes} $L_R := \sum_{t=2}^T \pmb{1}\{ \mu_t(a_{t-1}) < \mu_{t-1}(a_{t-1})\}$.

To contrast, the prior works \citet{kim22,kim24} consider a priori upper bounds on $V, L$ (i.e., the adversary is constrained to incur non-stationarity at most size $V$ or count $L$), and so only show expected regret bounds in terms of such bounds.
Our work establishes stronger high-probability regret bounds in terms of our tighter realized values of $V, L$.

\section{Regret Lower Bounds}\label{sec:lower}

\citet[Theorem 4.1 and 4.2]{kim24} show regret lower bounds of order $(L_R+1)^{\frac{1}{\beta+1}} \cdot T^{\frac{\beta}{\beta+1}} \land ( V_R^{\frac{1}{\beta+2}} \cdot T^{\frac{\beta+1}{\beta+2}} + T^{\frac{\beta}{\beta+1}} )$ for the rotting infinite-armed bandit problem, which is a sub-case of our non-stationary setup.
Our regret upper bound in \Cref{cor:elim-bounds} matches this lower bound up to $\log$ factors, without algorithmic knowledge of $L_R,V_R$.

\section{A Blackbox for Optimally Tracking Unknown Non-Stationarity}\label{sec:blackbox}

\subsection{Intuition for Subsampling}\label{subsec:subsampling}

A key idea, used for the stationary problem in \citet{wang08,bayati20}, is that of {\em subsampling} a fixed set of arms from the reservoir.
The main algorithmic design principle is to run a finite-armed MAB algorithm over this subsample.
The choice of subsample size is key here and exhibits its own exploration-exploitation tradeoff, appearing through a natural regret decomposition with respect to the subsample, which we'll denote by $\mc{A}_0 \subseteq \mc{A}$:
\begin{equation}\label{eq:regret-decomp}
	\bld{R}_T = \underbrace{ \sum_{t=1}^T  \min_{a \in \mc{A}_0} \delta_t(a) }_{\text{Regret of best subsampled arm}}  + \underbrace{ \sum_{t=1}^T  \max_{a \in \mc{A}_0} \delta_t(a,a_t)}_{\text{Regret to best subsampled arm}}.
\end{equation}
Suppose there are $K := |\mc{A}_0|$ subsampled arms.
One can show (e.g., \Cref{lem:best-initial-arm}) a size $K$ subsample of a $\beta$-regular reservoir contains, with high probability, an arm with gap $O(K^{- 1/\beta})$.
Thus, the first sum in \Cref{eq:regret-decomp} is at most $T\cdot K^{-1/\beta}$.

Then, plugging in the classical gap-dependent regret bounds for finite-armed MAB, the second sum in \Cref{eq:regret-decomp} is $\tilde{O}\left( \sum_{i=2}^K \Delta_{(i)}^{-1} \right)$ where $\Delta_{(i)}$ is the (random) $i$-th smallest gap to the best subsampled arm.
Then, it is further shown \citep[Section A.2]{bayati20} that this gap-dependent quantity scales like $\tilde{O}(K)$ in expectation, by carefully integrating $\Delta_{(i)}^{-1}$ over the randomness of the reservoir.

Then, choosing $K$ to balance the bounds $T \cdot K^{-1/\beta}$ and $K$ on \Cref{eq:regret-decomp} yields an optimal choice of $K \propto T^{\frac{\beta}{\beta+1}}$ giving a regret bound of $T^{\frac{\beta}{\beta+1}}$ which is in fact minimax.

\bld{Key Challenges:}
Extending this strategy to the non-stationary problem, it's natural to ask if we can follow an analogous strategy by reducing to $K$-armed non-stationary bandits.
However, this poses fundamental difficulties:
\vspace{-1em}
\begin{enumerate}[(a)]
	\item As our goal in the non-stationary problem is to achieve {\em adaptive} regret bounds, without parameter knowledge, a naive approach is to reduce to adaptive $K$-armed non-stationary MAB guarantees \citep{auer2019,wei2021,suk22,abbasi22}.
	However, these guarantees only hold for an oblivious adversary, and so are inapplicable to our problem.
	Furthermore, these algorithms only give worst-case rates of the form $\sqrt{LKT}$ in terms of $L$ changes.
	In fact, it's known in this literature that no algorithm can adaptively secure tighter gap-dependent rates over unknown changepoints \citep[Theorem 13]{garivier2011}.
	As the gap-dependent regret bound is crucial to achieving optimally balancing exploration and exploitation in our subsampling strategy, we see this approach can only hope to achieve suboptimal rates.

	\item Secondly, we observe that upon experiencing changes, one may have to {\em re-sample} arms from the reservoir distribution as the regret of the best subsampled arm $\min_{a \in \mc{A}_0} \delta_t(a)$ can itself become large over time. 
	Thus, we require a more refined subsampling strategy which works in tandem with non-stationarity detection.
\end{enumerate}

\subsection{Our New Approach: Regret Tracking}\label{subsec:our-approach}

We handle both of the above issues with the new idea of tracking the empirical regret $\hat{\delta}_t(a_t) := 1 - Y_t(a_t)$ as a proxy for tracking non-stationarity.
The key observation is that the empirical cumulative regret of played actions up to round $t$, $\sum_{s=1}^t \hat{\delta}_s(a_s)$, concentrates around $\sum_{s=1}^t \delta_s(a_s)$ at fast logarithmic rates by Freedman's inequality and a self-bounding argument for $[0,1]$-valued random variables.

This means, so long as $\sum_{s=1}^t 1 - Y_s(a_s) \lesssim t^{\frac{\beta}{\beta+1}}$, our regret will be safe up to the minimax stationary regret rate.
On the other hand, if $\sum_{s=1}^t 1 - Y_s(a_s) \gg t^{\frac{\beta}{\beta+1}}$, then the agent must be experiencing large regret which means some non-stationarity has occurred if the agent otherwise plays optimally for stationary environments. 


Thus, at a high level, our procedure (\Cref{alg:blackbox})
restarts the subsampling strategy outlined in \Cref{subsec:subsampling} upon detecting large empirical regret.

Setting up relevant terminology, an {\em episode} is the set of rounds between consecutive restarts and, within each episode, we further employ doubling epochs, termed {\em blocks}, to account for unknown changepoints and durations of play.

Within each block, we run the subsampling strategy for a fixed time horizon as a blackbox.
The blackbox takes as input a finite-armed MAB base algorithm, parametrized by $\base(t,\mc{A}_0)$ for inputs horizon $t$ and subsampled set of arms $\mc{A}_0 \subset \mc{A}$.

Our only requirement of the base algorithm is that it attains a gap-dependent regret bound in so-called {\em mildly corrupt} environments, defined below.
It's straightforward to show this is satisfied by classical stochastic MAB algorithms such as UCB \citep{lai} (proof in \Cref{app:base}). 

\begin{defn}\label{defn:mild}
	We say a finite-armed non-stationary bandit environment $\{ \mu_t(a)\}_{t \in [T], a \in \mc{A}_0}$ over horizon $T$ with armset $\mc{A}_0$ is {\bf $\alpha$-mildly corrupt} for $\alpha>0$ if there exists a {\bf reference} reward profile $\{ \mu(a)\}_{a \in \mc{A}_0}$ such that
\[
	\forall t \in [T] , a \in \mc{A}_0: |\mu_t(a) - \mu(a) | \leq \alpha . 
\]
\end{defn}

Next, in stating the requirement of our base algorithm, we use $\delta_t(a,a') := \mu_t(a) - \mu_t(a')$ to denote the gap of arm $a'$ to $a$ in the context of a finite-armed bandit $\{\mu_t(a)\}_{t\in [T], a \in \mc{A}_0}$.

\begin{assumption}\label{assumption:base}
	Let $\{\mu_t(a)\}_{t\in [T],a\in \mc{A}_0}$ be an $\alpha$-mildly corrupt $T$-round finite-armed bandit with reference reward profile $\{\mu(a)\}_{a \in \mc{A}_0}$.
 	Let $\Delta_{(2)} \leq \Delta_{(3)} \leq \cdots \Delta_{(|\mc{A}_0|)}$ be the ordered gaps induced by the reference reward profile.
	Then, $\base(T,\mc{A}_0)$ attains, with probability at least $1-1/T$, for all $t \in [T]$, a $t$-round static regret bound of (for $C_0$ free of $t,T,\mc{A}_0$):
	\vspace{-0.5em}
	\begin{align*}
		\max_{a\in \mc{A}_0} \sum_{s=1}^t \delta_s(a,a_s) \leq 6t \alpha + C_0\sum_{i=2}^{|\mc{A}_0|} \frac{\log(T)}{\Delta_{(i)}} \pmb{1}\{ \Delta_{(i)} \geq 4\alpha \} ,
	\end{align*}
\end{assumption}

\begin{rmk}
	Our \Cref{assumption:base} is curiously similar to Assumption 1 of \citet{wei2021}, but it in fact stronger in requiring gap-dependent bounds in environments with small variation as opposed to their requirement of $O(\sqrt{T})$ regret in such environments \citep[Lemma 3]{wang22}.
\end{rmk}

\begin{rmk}
	Bandit algorithms attaining state-of-the-art regret bounds in {\bf stochastic regimes with adversarial corruption} \citep{lykouris18,gupta19,zimmert19,ito21,ito23,dann23} satisfy \Cref{assumption:base}. 
\end{rmk}

\IncMargin{1em}
\begin{algorithm2e}
 	\caption{{Blackbox Non-Stationary Algorithm}}
 \label{alg:blackbox}
 \bld{Input}: Finite-armed MAB algorithm $\base$ satisfying \Cref{ass:reservoir}. Subsampling rate $S_m$.\\
 \bld{Initialize}: Episode count $\ell \leftarrow 1$, Starting time $t_1^1 \leftarrow 1$.\\
 \For{$m=1,2,\ldots,\ceil{\log(T)}$}{ \label{line:doubling}
	 Subsample $S_m \land 2^m$  arms $\mc{A}_m \subset \mc{A}$.\\ 
		 Initiate a new instance of $\base(2^m,\mc{A}_m)$.\\ 
		 \For{$t = t_{\ell}^m,\ldots,(t_{\ell}^m + 2^m-1) \land T$}{
			 Play arm $a_t$ (receiving reward $Y_t(a_t)$) as chosen by $\base(2^m,\mc{A}_m)$.\\

			 \bld{Changepoint Test:} \If{$\ds\sum_{s = t_{\ell}^m}^t \hat{\delta}_s(a_s) \geq C_1 \cdot ( |\mc{A}_m| \vee 2^{m/2}) \cdot \log^3(T) $}{ \label{line:cpd-test}
				 Restart: $t_{\ell+1}^1 \leftarrow t+1$, $\ell \leftarrow \ell+1$. \label{line:restart}\\
				 Return to \Cref{line:doubling} (Restart from $m=1$).
		}
        \ElseIf{ $t=t_\ell^m+2^m-1$}{$t_{\ell}^{m+1}\leftarrow t+1$ (Start of the $(m+1)$-th {\em block} in the $\ell$-th episode).}
	}
}
\end{algorithm2e}
\DecMargin{1em}

\subsection{Blackbox Regret Upper Bound}\label{subsec:blackbox-regret}

The main result of this section is that \Cref{alg:blackbox} attains the optimal regret in terms of number of changes $L$ and total-variation $V$ when $\beta \geq 1$ and matches the state-of-art regret bounds with known $L,V$ for $\beta < 1$ \citep{kim24}.

\begin{thm}\label{thm:blackbox}
	Under \Cref{ass:reservoir} with $\beta \geq 1$, \Cref{alg:blackbox} with $S_m := \ceil{ 2^{m \cdot \frac{\beta}{\beta+1}} }$ satisfies, w.p. $1-O(1/T)$:
	\[
		{\normalfont \bld{R}}_T \leq \tilde{O} \left( (L+1)^{\frac{1}{\beta+1}}  T^{\frac{\beta}{\beta+1}} \land (V^{\frac{1}{\beta+2}} T^{\frac{\beta+1}{\beta+2}} + T^{\frac{\beta}{\beta+1}}) \right) .
	\]
	If $\beta < 1$, \Cref{alg:blackbox} with $S_m := \ceil{  2^{m \cdot \beta/2}  }$ satisfies w.p. $1 - O(1/T)$:
	\[
		{\normalfont \bld{R}}_T \leq \tilde{O} \left( \sqrt{(L+1) \cdot T} \land ( V^{1/3} T^{2/3} + \sqrt{T} ) \right).
	\]
\end{thm}

\begin{proof}{(Outline)}
	We given an outline of the proof with full details deferred to \Cref{app:blackbox}.
	We also focus on the setting of $\beta \geq 1$ with (minor) modifications of the argument for $\beta < 1$ discussed in \Cref{subsec:beta-small}.
	Let $t_{\ell} := t_{\ell}^1$ be the start of the $\ell$-th episode $\ho{t_{\ell}}{t_{\ell+1}}$.
	Let $\hat{L}$ be the (random) number of restarts triggered over $T$ rounds. Let $m_\ell$ be the index of the last block in $\ell$-th episode.

	\paragraph*{$\bullet$ Converting Empirical Regret Bound to Per-Episode Regret Bound.}

	Following the discussion of \Cref{subsec:our-approach}, we first use concentration to upper bound the per-block regret $\sum_{s=t_{\ell}^m}^{t_{\ell}^{m+1}-1} \delta_s(a_s)$ on each block and to also lower bound it on blocks concluding with a restart.
	We first have by Freedman's inequality (\Cref{lem:freedman}) and AM-GM, with high probability, for all subintervals $[s_1,s_2]$ of rounds:
	\begin{align}
		\left| \sum_{s=s_1}^{s_2} \delta_s(a_s) - \hat{\delta}_s(a_s) \right| &\lesssim \sqrt{\log(T) \sum_{s=s_1}^{s_2} \delta_s(a_s)} + \log(T) 
										      \lesssim \half \sum_{s=s_1}^{s_2} \delta_s(a_s) + \log(T) \numberthis \label{eq:concentration-body}
	\end{align}
	This allows us to upper bound the regret on each block $\ho{t_{\ell}^m}{t_{\ell}^{m+1}}$ by $\tilde{O}(S_m)$ using the bound on empirical regret $\sum_{s=t_{\ell}^m}^{t_{\ell}^{m+1}-1} \hat{\delta}_s(a_s)$ from \Cref{line:cpd-test} of \Cref{alg:blackbox}.


	Then, summing $S_m \propto 2^{m \cdot \frac{\beta}{\beta+1}}$ over blocks and episodes yields a total regret bound of $\tilde{O} \left( \sum_{\ell=1}^{\hat{L}} (t_{\ell+1} - t_{\ell})^{\frac{\beta}{\beta+1}} \right)$. 

\paragraph*{$\bullet$ Bounding the Variation Over Each Episode.}
It now remains to relate
	$
		\sum_{\ell=1}^{\hat{L}} (t_{\ell+1} - t_{\ell})^{\frac{\beta}{\beta+1}}
	$
	to the total count $L$ of changes and total-variation $V$.
	To this end, we show there is a minimal amount of variation in each episode $\ho{t_{\ell}}{t_{\ell+1}}$ which will allow us to conclude the total regret bound using arguments similar to prior works \citep[Corollary 2]{suk22} \citep[Lemma 5]{chen2019}.

	We first introduce a regret decomposition, alluded to earlier in \Cref{eq:regret-decomp}, based on the ``best initially subsampled arm'', $\widehat{a}_{\ell,m_\ell} := \argmax_{a\in \mc{A}_{m_{\ell}}} \mu_0(a)$, where we use $\mu_0(a)$ to denote the initial reward of arm $a$ when it's first sampled.
	The regret in the last block $\ho{t_{\ell}^{m_{\ell}},t_{\ell+1}}$ of the $\ell$-th episode is:
\begin{align*}
	\sum_{t=t_{\ell}^{m_\ell}}^{t_{\ell+1}-1} \delta_t(a_t) = \underbrace{ \sum_{t=t_\ell^{m_\ell}}^{t_{\ell+1}-1} \delta_{t}(\hat{a}_{\ell,m_\ell}) }_{\mytag{(A)}{case1}} + \underbrace{ \sum_{t=t_{\ell}^{m_\ell}}^{t_{\ell+1}-1} \delta_t(\hat{a}_{\ell,m_\ell}, a_t)}_{\mytag{(B)}{case2}}
\end{align*}
	Then from the above and concentration, one of two possible cases must hold if a restart occurs on round $t_{\ell+1}$: either
	    \ref{case1} or \ref{case2} is $\Omega(S_{m_{\ell}})$.
	In either case, we claim large variation occurs over the episode:
	\begin{equation}\label{eq:large-variation-ep}
		\sum_{t=t_{\ell}}^{t_{\ell+1}-1} |\mu_{t}(a_{t-1}) - \mu_{t-1}(a_{t-1})| \geq (t_{\ell+1} - t_{\ell})^{-\frac{1}{\beta+1}}.
	\end{equation}
\paragraph*{$\bullet$ Regret of Best Subsampled Arm is Large.}
In the case where \ref{case1} is $\Omega( S_{m_{\ell}})$, we know due to our subsampling rate $S_{m_{\ell}}$ that $\hat{a}_{\ell, m_{\ell}}$ will w.h.p. have an initial gap of $\tilde{O}( 2^{- m_{\ell} \cdot \frac{1}{\beta+1}})$ (\Cref{lem:best-initial-arm}).
On the other hand, \ref{case1} being large means there's a round $t' \in \ho{t_{\ell}^m}{t_{\ell+1}}$ such that
\[
	\delta_{t'}(\hat{a}_{\ell,m_{\ell}}) \gtrsim S_{m_{\ell}} / (t_{\ell+1} - t_{\ell}^{m_{\ell}}) \gtrsim 2^{- m_{\ell} \frac{1}{\beta+1}}.
\]
Thus, from $2^{- m_{\ell} \frac{1}{\beta+1}} \geq (t_{\ell+1} - t_{\ell})^{-\frac{1}{\beta+1}}$, \Cref{eq:large-variation-ep} holds.

\paragraph*{$\bullet$ Regret of Base is Large.}
Now, if \ref{case2} is $\Omega(S_{m_{\ell}})$ but \ref{case1} is $o(S_{m_{\ell}})$, suppose for contradiction that \Cref{eq:large-variation-ep} is reversed.
Then, this means the finite MAB environment experienced by the base algorithm is $(t_{\ell+1} - t_{\ell})^{-\frac{1}{\beta+1}}$-mildly corrupt (\Cref{defn:mild}).
Thus, \Cref{assumption:base} bounds the regret of the base, which in turn bounds the per-block regret above: 
\begin{align}
\sum_{s=t_{\ell}^{m_{\ell}}}^{t_{\ell+1}-1} &\delta_t(\hat{a}_{\ell,m_{\ell}}, a_t) \lesssim t^{\frac{1}{\beta+1}} +  \sum_{i=2}^{S_{m_{\ell}}} \frac{\log(T)}{\Delta_{(i)}} \pmb{1}\left\{ \frac{\Delta_{(i)}}{4} \geq t^{- \frac{\beta}{\beta+1}} \right\}  \label{eq:subsample-bound-block}.
\end{align}
where $\{\Delta_{(i)}\}_{i = 1}^{S_{m_{\ell}}}$ are the ordered initial gaps to $\hat{a}_{\ell,m_{\ell}}$ of the arms in $\mc{A}_{m_{\ell}}$.

We then bound the RHS of \Cref{eq:subsample-bound-block} by $O(S_{m_{\ell}})$ to contradict our premise that \ref{case2} is $\Omega(S_{m_{\ell}})$.
In \citet[Lemma D.4]{bayati20}, bounding \Cref{eq:subsample-bound-block} by $O(S_{m_{\ell}})$ was done in expectation by carefully integrating the densities of each random variable $\Delta_{(i)}^{-1}$.
But, this requires additional regularity conditions on an assumed density for the reservoir distribution when $\beta \neq 1$ (op. cit. Section 6.1).
To contrast, we bound \Cref{eq:subsample-bound-block} in high probability using a novel binning argument on the values of the $\Delta_{(i)}$'s and then using Freedman-type concentration on the number of subsampled arms within each bin.
Details of this can be found in \Cref{subsec:proof-variation-bound}.
\end{proof}

\section{Tracking Significant Shifts}\label{sec:sig-shifts}

While \Cref{alg:blackbox} is a flexible blackbox which can make use of any reasonable finite-armed MAB algorithm, the regret bound of \Cref{thm:blackbox} is suboptimal for $\beta < 1$.
Futhermore, we aim to show regret upper bounds in terms of the tighter rotting measures of non-stationarity $L_R$ and $V_R$ (cf. \Cref{subsec:nonstat}).

In fact, we go beyond this aim and define a new measure of non-stationarity which precisely tracks the rotting changes which are most severe to performance.

\subsection{Defining a Significant Shift}

Recent works on non-stationary finite-armed MAB achieve regret bounds in terms of more nuanced non-stationarity measures, which only track the most {\em significant} switches in best arm, or those truly necessitating re-exploration \citep{suk22,buening22,sukagarwal23,suk23,suk24}.

Here, we develop and study an analogous notion for the infinite-armed bandit.
First, we note in the infinite-armed bandit problem, there's no single ``best arm'', as the arm-space is infinite and, almost surely, no sampled arm will have the optimal reward value of $1$.
Yet, the core inution behind the notion of significant shift for finite MAB \citep[Definition 1]{suk22} remains by definining a significant shift in terms of the regret experienced by each arm.
First, we say an arm $a$ is {\em safe} on an interval $[s_1,s_2]$ of rounds if:
\begin{equation}\label{eq:sig-regret}
\sum_{s=s_1}^{s_2} 1 - \mu_s(a) \leq \kappa_1^{-1} \cdot (s_2 - s_1 + 1)^{\frac{\beta}{\beta+1}} .
\end{equation}


Then, a significant shift is roughly defined as occurring when every arm in some set of arms $\hat{\mc{A}}$ is unsafe and violates \Cref{eq:sig-regret}.
If we let $\hat{\mc{A}} = \mc{A}$ be the full reservoir set of arms, then this proposed definition may not by meaningful as there could always exist an unsampled arm which is safe in terms of regret, but unknown to the agent.

For the subsampling strategy discussed in \Cref{subsec:subsampling}, we argue it is sensible to let $\hat{\mc{A}} = \mc{A}_0$, the set of sub-sampled arms. 
Based on the discussion of \Cref{sec:blackbox}, we see that a subsample of size $\Omega( t^{\frac{\beta}{\beta+1}} )$ contains w.h.p. a safe arm with gap $O(t^{-\frac{1}{\beta+1}})$.
This motivates the following definition of significant shift, which is defined for any agent (regardless of whether a subsampling strategy is used).

\begin{defn}[Significant Shift]\label{defn:sig-shift}
	We say a bandit environment over $t$ rounds is {\bf safe} if there exists an arm $a$ among the first $t^{\frac{\beta}{\beta+1}}$ arms sampled by the agent such that \Cref{eq:sig-regret} holds for all intervals of rounds $[s_1,s_2] \subseteq [t]$. 
    Let $\tau_0 := 1$.
    Define the $(i+1)$-th {\bf significant shift} $\tau_{i+1}$ given $\tau_i$ as the first round $t > \tau_i$ such that the bandit environment over rounds $[\tau_i,t]$ is not safe.
	Let $\Lsig$ be the largest significant shift over $T$ rounds and by convention let $\tau_{\Lsig+1} := T+1$.
\end{defn}

\subsection{Comparing Significant Shifts with $L_R,V_R$}
We first note that the significant shift, like $L_R$ and $V_R$, is an agent-based measure of non-stationarity and so depends on the agent's past decisions.

Next, we caution that a significant shift does not always measure non-stationarity and can in fact be triggered in stationary environments for a ``bad'' algorithm.
For example, consider an agent sampling exactly one arm from the reservoir and committing to it.
Then, due to the large variance of a single sample, such an algorithm incurs constant regret with constant probability.
In this case, a significant shift is triggered even in a stationary environment as none of the subsample is safe.
Thus, in this example, the significant shift tracks the agent's suboptimality.

In spite of this example, we argue that \Cref{defn:sig-shift} is still a meaningful measure of non-stationarity for all algorithms of theoretical interest, and indeed learning significant shifts via \Cref{alg:elim} will allow us to attain optimal rates w.r.t. $L_R,V_R$ (\Cref{cor:elim-bounds}).
Note that any algorithm attaining the optimal high-probability regret of $O(T^{\frac{\beta}{\beta+1}})$ in stationary environments must sample at least $T^{\frac{\beta}{\beta+1}}$ arms.
Intuitively, this is because securing an arm with gap at most $\delta := T^{-\frac{1}{\beta+1}}$ (which occurs for a given arm with probability $\mb{P}(\mu_0(a) > 1 - \delta) = \Theta(\delta^{\beta})$) requires $\Omega(\delta^{-\beta})$ trials.
Indeed, this claim is also seen in the proofs of the stationary regret lower bounds \citep[e.g.,][Theorem 3]{wang08}.

At the same time, any sample of $\Omega(T^{\frac{\beta}{\beta+1}})$ arms from the reservoir will contain an arm with gap $O(\delta)$ w.h.p.. (\Cref{lem:best-initial-arm}), which will maintain \Cref{eq:sig-regret} over all subintervals.
Then, a significant shift will not be triggered unless this best subsampled arm' gap increases.
Thus, any algorithm attaining optimal regret in stationary environments will satisfy w.h.p. $\Lsig \leq L$.

Additionally, the only way an arm with initial gap $O(\delta)$ becomes unsafe and violates \Cref{eq:sig-regret} is if there's large rotting total variation (i.e., $V_R$ is bounded below).
This means $\Lsig \leq L_R$ and we'll see in the next subsection that learning significant shifts in fact allows us to recover the optimal rate $V_R^{\frac{1}{\beta+2}} \cdot T^{\frac{\beta+1}{\beta+2}}$ in terms of $V_R$ (\Cref{cor:elim-bounds}).



\subsection{Regret Upper Bounds}\label{subsec:regret-upper-bound-elim}

Here, we derive a tight and adaptive regret bound in terms of the significant shifts.
In particular, we show a regret bound of $(\Lsig + 1)^{\frac{1}{\beta+1}}  T^{\frac{\beta}{\beta+1}} \land ( V_R^{\frac{1}{\beta+2}}  T^{\frac{\beta+1}{\beta+2}} + T^{\frac{\beta}{\beta+1}} )$ in terms of $\Lsig$ significant shifts and $V_R$ total rotting variation.

A preliminary task here is as follows:
\vspace{-0.25em}
\begin{goal}\label{goal}
	Show a regret bound of $\tilde{O}(t^{\frac{\beta}{\beta+1}})$ in $t$-round safe environments.
\end{goal}
\vspace{-1em}
Given a base procedure achieving said goal, we can then use a restart strategy similar to \Cref{alg:blackbox} where we restart the base upon detecting a significant shift.


For the finite $K$-armed setting, the analogous preliminary claim \citep[Theorem 11]{suk24} is to show a regret bound of $O(\sqrt{KT})$ in safe environments.
This is achieved using a randomized variant of successive elimination \citep{even-dar}.

However, once again the infinite-armed setting is more nuanced and so this claim cannot directly be applied.
Indeed, setting $K := T^{\frac{\beta}{\beta+1}}$ results in a suboptimal rate of $\sqrt{KT} = T^{ \frac{ \beta + 1/2}{\beta+1}}$.
The fundamental issue here is that the $\sqrt{KT}$ rate captures a worst-case variance of estimating bounded rewards.
To get around this, we do a more refined variance-aware regret analysis relying on self-bounding techniques similar to those used to show \Cref{thm:blackbox} in \Cref{sec:blackbox}.

To further emphasize the more challenging nature of showing \Cref{goal}, we notice that the regret analysis of our blackbox procedure \Cref{alg:blackbox} in \Cref{subsec:blackbox-regret} crucially relies on bounding the logarithmic gap-dependent regret rate of \Cref{assumption:base} over periods of small total-variation (used to show \Cref{eq:large-variation-ep}).
However, such a gap-dependent regret rate is ill-defined when the gaps $\Delta_i$ can be changing substantially over time (as to violate \Cref{assumption:base}) which can happen in safe environments while we expect stationary regret rates.

Nevertheless, we show that a different per-arm regret analysis gets around this difficulty.
Our procedure (\Cref{alg:elim}) is a restarting randomized elimination  using the same doubling block scheme of \Cref{alg:blackbox}.


Going into more detail, we note, by uniformly exploring a {\em candidate armset} $\mc{G}_t$ at round $t$, we can maintain importance-weighted estimates of the gaps of each arm $a \in \mc{G}_t$:
\[
	\IW_t(a) := \frac{(1 - Y_t(a_t)) \cdot \pmb{1}\{ a_t = a\}}{\mb{P}(a_t = a \mid \mc{H}_{t-1})},
\]
where $\mc{H}_{t-1}$ is the $\sigma$-algebra generated by decisions and observations up to round $t-1$.

Next, we note, by Freedman's inequality that the estimation error of the cumulative estimate $\sum_{t \in I} \IW_t(a)$ over an interval $I$ of rounds scales like $\sqrt{\sum_{t \in I} \delta_t(a) \cdot |\mc{G}_t|}$.
The inclusion of the $\delta_t(a)$ term inside of the square root is crucial here and, using self-bounding arguments, yield tighter concentration bounds of order $\tilde{O}(\max_{t\in I} |\mc{G}_t|)$, which we can use as a threshold for fast variance-based elimination.


\IncMargin{1em}
\begin{algorithm2e}[h]
	\caption{{Restarting Subsampling Elimination}}
 \label{alg:elim}
 \bld{Initialize}:  Episode count $\ell \leftarrow 1$, start $t_1^1 \leftarrow 1$.\\
 \For{$m=1,2,\ldots,\ceil{\log(T)}$}{ \label{line:restart-elim}
	 Subsample $\ceil{2^{(m+1) \cdot \frac{\beta}{\beta+1}} \log(T) } \land 2^m$  
	 arms $\mc{A}_m$ and let $\mc{G}_{t_\ell^m} \leftarrow \mc{A}_m$. \label{line:subsample} \\
		 \For{$t = t_\ell^m,\ldots,(t_\ell^m + 2^m-1)\land T$}{
			 Play arm $a_t$ as $\Unif\{\mc{G}_t\}$ and observe reward $Y_t(a_t)$.\\
			 Eliminate arms: $\mc{G}_{t+1} \leftarrow \mc{G}_t \bs \left\{ a : \ds\sum_{s=t_\ell^m}^t \IW_s(a) \geq C_2 \cdot |\mc{A}_m| \log(T) \right\}$. \label{line:elim} \\
			 \bld{Restart Test:} \uIf{$\mc{G}_{t+1} = \emptyset$}{
		 Restart: $t_{\ell+1}^1 \leftarrow t+1$, $\ell\leftarrow \ell+1$.\\
		 Return to \Cref{line:restart-elim} (Restart from $m=1$).
		}
        \ElseIf{ $t=t_\ell^m+2^m-1$}{$t_{\ell}^{m+1}\leftarrow t+1$ (Start of the $m+1$-th {\em block} in the $\ell$-th episode).}
	}
}
\end{algorithm2e}
\DecMargin{1em}

\begin{thm}\label{thm:regret-episodes}
	Let $\hat{L}$ be the number of episodes $\ho{t_{\ell}}{t_{\ell+1}}$ elapsed in \Cref{alg:elim} over $T$ rounds.
	\Cref{alg:elim} satisfies, w. p. at least $1-1/T$:
	\[
		{\normalfont \bld{R}}_T =\tilde{O} \left(\sum_{\ell=1}^{\hat{L}} ( t_{\ell+1} -t_{\ell} )^{\frac{\beta}{\beta+1}} \right).
	\]
\end{thm}

\begin{proof}{(Outline)}
    %
	We give a proof outline here and full details are found in \Cref{app:sigshift-details}.
	It suffices to bound the regret on block $\ho{t_{\ell}^m}{t_{\ell}^{m+1}}$ by $\tilde{O}(S_m)$.
	To show this, we do a variance-aware version of the per-arm regret analysis of Section B.1 in \citet{suk22}. 

	We first transform the regret according to its conditional expectation.
	Note that, from the uniform sampling strategy,
	\[
		\mb{E}[ \delta_t(a_t) \mid \mc{H}_{t-1}] = \sum_{a \in \mc{G}_t} \frac{\delta_t(a)}{|\mc{G}_t|}.
	\]
	Next, note that
	\[
		\Var[ \delta_t(a_t) \mid \mc{H}_{t-1} ] \leq \mb{E}[ \delta_t^2(a_t) \mid \mc{H}_{t-1} ] \leq \mb{E}[ \delta_t(a_t) \mid \mc{H}_{t-1} ].
	\]
	Then, using Freedman's inequality (\Cref{lem:freedman}) with the above, we have for all subintervals $I \subseteq [T]$, with probability at least $1-1/T$: 
	\begin{align*}
		\sum_{t\in I} \delta_t(a_t) - \sum_{a \in \mc{G}_t} \frac{\delta_t(a)}{|\mc{G}_t|} &\lesssim  \sqrt{ \log(T) \sum_{t\in I} \mb{E}[ \delta_t(a_t) \mid \mc{H}_{t-1}] }  
		+ \log(T) 
		\lesssim \sum_{t\in I} \sum_{a \in \mc{G}_t} \frac{\delta_t(a)}{|\mc{G}_t|} + \log(T),
	\end{align*}
	where the second inequality is from AM-GM.
	In light of the above, it remains to bound $\sum_{a=1}^K \sum_{t=t_{\ell}^m}^{t^a} \frac{\delta_t(a)}{|\mc{G}_t|}$
	where $t^a$ is the last round in block $\ho{t_{\ell}^m}{t_{\ell}^{m+1}}$ that $a$ is retained.

	We next again use Freedman's inequality (\Cref{lem:freedman}) and self-bounding to relate $\sum_{t\in I} \delta_t(a)$ to $\sum_{t\in I} \IW_t(a)$.
	We have w.p. at least $1-1/T$:
    \begin{align}
	    \left|\sum_{t=t_{\ell}^m}^{t^a} \delta_t(a) - \IW_t(a) \right| 
	     &\lesssim  \max_{t\in [t_{\ell}^m,t_a]} |\mc{G}_t| \log(T) + \sqrt{ \log(T) \sum_{t=t_{\ell}^m}^{t^a} \delta_t(a) \cdot |\mc{G}_t|} \nonumber \\
	    &\le  \half \sum_{t=t_{\ell}^m}^{t^a} \delta_t(a) + C' \max_{t \in [t_{\ell}^m,t^a]} |\mc{G}_t| \log(T),\label{eq:delta_error_bd_body}
    \end{align}
	where again we use AM-GM in the second inequality.
	Next,
	\begin{align*}
		\sum_{t=t_{\ell}^m}^{t^a} \frac{\delta_t(a)}{|\mc{G}_t|} &\leq \left( \max_{s \in [t_{\ell}^m,t^a]} \frac{1}{|\mc{G}_s|} \right) \sum_{t=t_{\ell}^m}^{t^a} \delta_t(a).
	\end{align*}
	Moving the $\half \sum_{t\in I} \delta_t(a)$ to the other side in \Cref{eq:delta_error_bd_body}, we get
	\[
		\sum_{t=t_{\ell}^m}^{t^a} \delta_t(a) \lesssim \sum_{t=t_{\ell}^m}^{t^a} \IW_t(a) + \log(T) \max_{t\in [t_{\ell}^m,t^a]} |\mc{G}_t|.
	\]
	Combining the above two displays with \Cref{line:elim} of \Cref{alg:elim}, we have:
	\begin{align*}
	    		\sum_{t=t_{\ell}^m}^{t^a-1} \frac{\delta_t(a)}{|\mc{G}_t|} \lesssim \max_{s\in [t_{\ell}^m,t^a-1]} \frac{S_m}{|\mc{G}_s|}  \log(T),
	\end{align*}
	Now, summing $\max_{s\in [t_{\ell}^m,t^a-1]} |\mc{G}_s|^{-1}$ over arms $a \in \mc{A}_m$ yields another $\log(S_m)$ factor, while summing over blocks $m\in[m_\ell]$ and episodes $\ell\in [\hat{L}]$ finishes the proof. 
\end{proof}

   We next show the regret bound of \Cref{thm:regret-episodes} in fact recovers the minimax regret rates in terms of $\Lsig$ and $V_T$. 
\begin{cor}\label{cor:elim-bounds}
	\Cref{alg:elim} satisfies, w.p. at least $1-1/T$:
    \begin{align*}
	    {\normalfont \bld{R}}_T \leq \tilde{O}\left( (\Lsig + 1)^{\frac{1}{\beta+1}}  T^{\frac{\beta}{\beta+1}} \land ( V_R^{\frac{1}{\beta+2}}  T^{\frac{\beta+1}{\beta+2}} + T^{\frac{\beta}{\beta+1}} ) \right).
    \end{align*}
\end{cor}

\begin{proof}{(Sketch)}
	Using similar arguments to the proof of \Cref{thm:blackbox}, within each block the best initial arm $\hat{a}_{\ell,m}$ has gap $\tilde{O}(2^{-m \cdot \frac{1}{\beta+1}})$.
	On the other hand, within a block $\ho{t_{\ell}^{m_{\ell}}}{t_{\ell+1}}$ ending in a restart, this best initial arm is eliminated meaning there exists a round $t \in \ho{t_{\ell}^{m_{\ell}}}{t_{\ell+1}}$ such that $\delta_t(\hat{a}_{\ell,m_{\ell}}) \gtrsim (t_{\ell+1} - t_{\ell}^{m_{\ell}})^{- \frac{1}{\beta+1}}$.
	Thus, the gap of $\hat{a}_{\ell,m_{\ell}}$ must have increased by at least this amount which gives a lower bound on the per-episode total variation of $\Omega( (t_{\ell+1} - t_{\ell})^{-\frac{1}{\beta+1}})$.
	Then, by similar arguments to the proof of \Cref{thm:blackbox},
	we deduce the total regret bound.
\end{proof}

\section{Comparing Blackbox vs. Elimination}\label{sec:discussion}

\Cref{cor:elim-bounds} and the nearly matching lower bounds of \Cref{sec:lower} show that only rotting non-stationarity ($\Lsig \leq L_R$ and $V_R$) factor into the difficulty of non-stationarity. 
In other words, rising non-stationarity is benign for non-stationary infinite-armed bandits.
Intuitively, this is because our problem assumes knowledge of both the top reward value and upper bound on rewards, which coincide and are equal to $1$.
Hence, arms with rising rewards require less exploration.

Interestingly, unlike the case with \Cref{alg:blackbox}, the elimination algorithm's bound in \Cref{cor:elim-bounds} does {\em not} require an upper bound on masses of the reservoir (\Cref{ass:reservoir}), 
or there's no dependence on $\kappa_2$ in the regret upper bounds of \Cref{cor:elim-bounds}.
This is new to this work and was not known even in the previous stationary regret bounds \citep{wang08,bayati20,kim24}.
This suggests our regret analysis is simpler and, indeed, we only require that the initial best arm in the subsample has small gap (which only requires lower bounded masses of the reservoir as we see in \Cref{lem:best-initial-arm}).
To contrast, the regret analysis of \Cref{thm:blackbox} (more similar to those of the aforementioned works) uses the upper bound on tail probabilities scaling with $\kappa_2$ to bound the regret of the finite-armed MAB base algorithm.
Such a step is avoided in the regret analysis of \Cref{thm:regret-episodes}  by estimating the regret of each arm separately.

On the other hand, the blackbox can be seen as more extensible as it allows for a wide range of finite-armed MAB algorithms to be used as a base.
We finally note \Cref{alg:elim} can essentially be reformulated as an instantiation of the blackbox with a successive elimination base algorithm.

\section{Experiments}
\begin{figure}[h]
\centering     
\subfigure[$\beta=0.8$]{\label{fig:a}\includegraphics[width=0.3\linewidth]{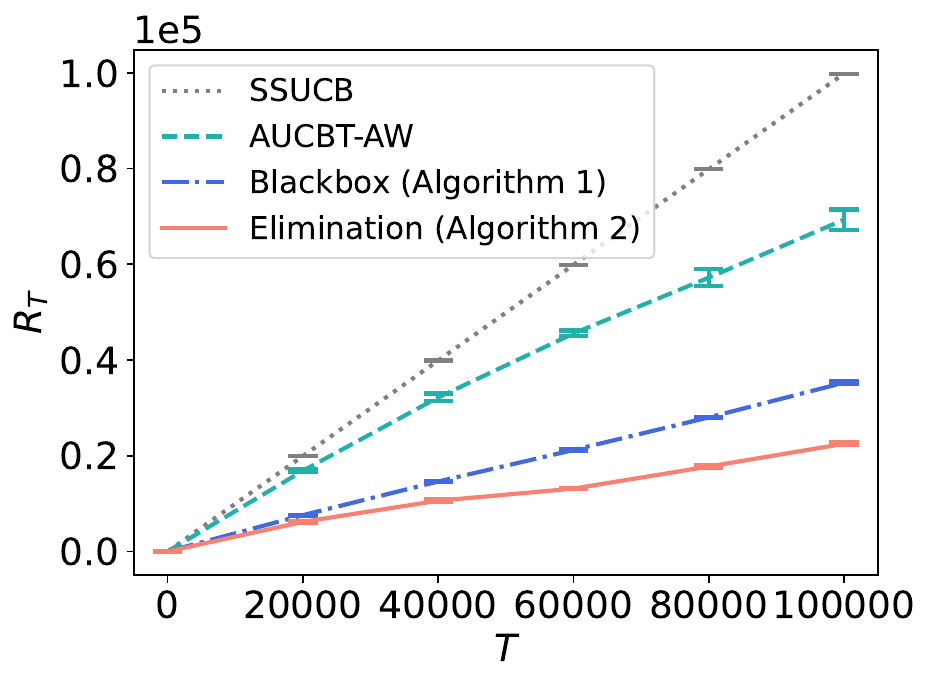}}
\subfigure[$\beta=1$]
 {\label{fig:b}\includegraphics[width=0.3\linewidth]{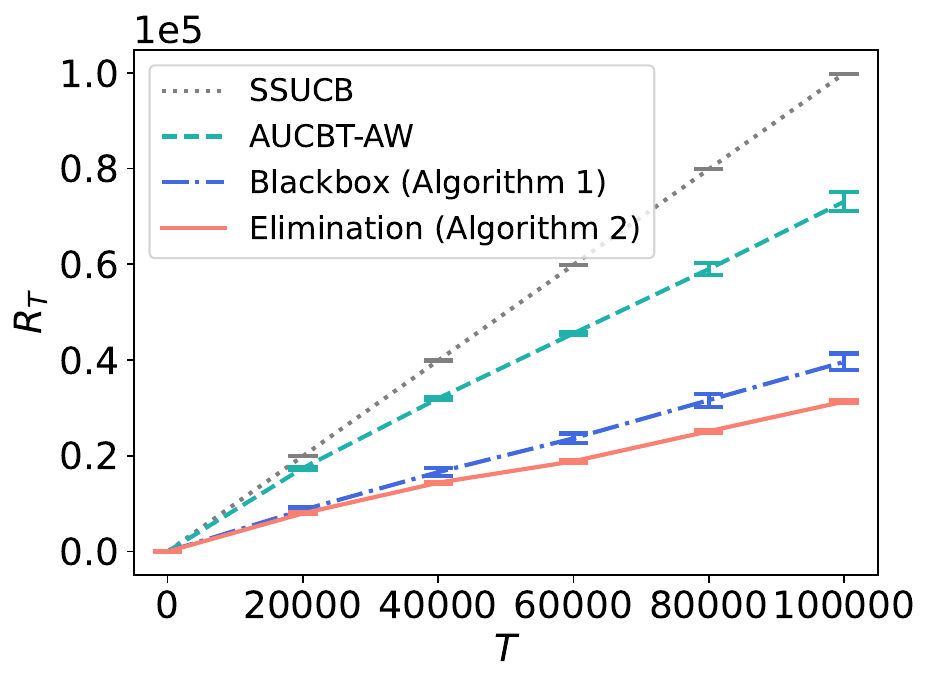}}
\subfigure[$\beta=1.2$]{\label{fig:c}\includegraphics[width=0.3\linewidth]{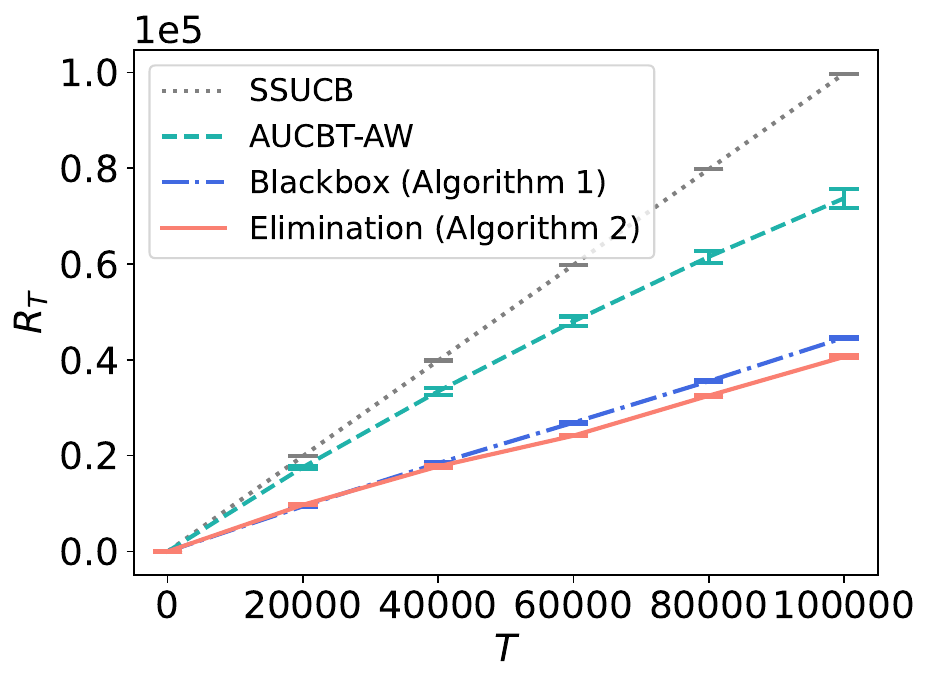}}
\vspace{-1mm}
\caption{Experimental results showing regret of algorithms}\label{fig:exp}\vspace{-1mm}
\end{figure}

Here, we demonstrate the performance of \Cref{alg:blackbox,alg:elim} on synthetic datasets.
For the $\base$ in \Cref{alg:blackbox} we use \texttt{UCB} \citep{auer2002finite}.
As benchmarks, we also implement \texttt{SSUCB} \citep{bayati20}, the optimal algorithm for stationary environments, and \texttt{AUCBT-ASW} \citep{kim24}, which achieves a suboptimal regret bound of  $
\tilde{O}(\min\{V^{\frac{1}{\beta+2}}T^{\frac{\beta+1}{\beta+2}},(L+1)^{\frac{1}{\beta+1}}T^{\frac{\beta}{\beta+1}}\} + \min\{T^{\frac{2\beta+1}{2\beta+2}},T^\frac{3}{4}\})$
in rested rotting scenarios.

To ensure a fair comparison between algorithms, we consider a rotting scenario where the mean reward of each selected arm decreases at a rate of \(1/t\) at time \(t\).
In this environment, for all our algorithms, it can be shown that \(L = \Omega(T) \), \(V = \tilde{O}(1)\).
For the case of \(\beta=1\) such that initial mean rewards follow a uniform distribution on $[0,1]$ (\Cref{fig:b}), both our algorithms outperform the benchmarks, with the elimination one achieving the best performance.
These results validate the insights of \Cref{sec:discussion}.
Specifically, our algorithms have a regret bound of $\tilde{O}(T^{2/3})$ (\Cref{thm:blackbox} and \Cref{cor:elim-bounds}) vs. \texttt{AUCBT-ASW}'s bound of $\tilde{O}(T^{3/4})$ \citep{kim24}.

In \Cref{fig:exp}, we observe that our algorithms outperform the benchmarks across various \(\beta\) values, aligning with theoretical results.
Furthermore, the performance gap between the elimination and blackbox algorithms increases as \(\beta\) decreases, which observation is also consistent with our theoretical results.


\bld{Note:} our implementation of \Cref{alg:elim} does not include the $\log(T)$ factor in the subsampling rate of \Cref{line:subsample}, as this lead to more stable experimental results.
	One can show this only changes the bound of \Cref{cor:elim-bounds} up to a $\log^{2/\beta}(T)$ factor, 
	which is not large for $\beta \in \{ 0.8, 1, 1.2\}$.




\bibliographystyle{plainnat}
\bibliography{bibs/bandit_general,
bibs/nonstat,
bibs/online,
bibs/duel,
bibs/contextual,
bibs/nonpar,
bibs/drift,
bibs/slow,
bibs/suk,
bibs/rot,
bibs/rl
}

\begin{thebibliography}{45}
\providecommand{\natexlab}[1]{#1}
\providecommand{\url}[1]{\texttt{#1}}
\expandafter\ifx\csname urlstyle\endcsname\relax
  \providecommand{\doi}[1]{doi: #1}\else
  \providecommand{\doi}{doi: \begingroup \urlstyle{rm}\Url}\fi

\bibitem[Abbasi-Yadkori et~al.(2023)Abbasi-Yadkori, György, and
  Lazić]{abbasi22}
Yasin Abbasi-Yadkori, András György, and Nevena Lazić.
\newblock A new look at dynamic regret for non-stationary stochastic bandits.
\newblock \emph{Journal of Machine Learning Research}, 24\penalty0
  (288):\penalty0 1--37, 2023.
\newblock URL \url{https://jmlr.org/papers/volume24/22-0387/22-0387.pdf}.

\bibitem[Auer et~al.(2002)Auer, Cesa-Bianchi, and Fischer]{auer2002finite}
Peter Auer, Nicolo Cesa-Bianchi, and Paul Fischer.
\newblock Finite-time analysis of the multiarmed bandit problem.
\newblock \emph{Machine learning}, 47\penalty0 (2-3):\penalty0 235--256, 2002.
\newblock URL
  \url{https://homes.di.unimi.it/~cesabian/Pubblicazioni/ml-02.pdf}.

\bibitem[Auer et~al.(2019)Auer, Gajane, and Ortner]{auer2019}
Peter Auer, Pratik Gajane, and Ronald Ortner.
\newblock Adaptively tracking the best bandit arm with an unknown number of
  distribution changes.
\newblock \emph{Conference on Learning Theory}, pages 138--158, 2019.
\newblock URL \url{https://proceedings.mlr.press/v99/auer19a/auer19a.pdf}.

\bibitem[Bayati et~al.(2020)Bayati, Hamidi, Johari, and Khosravi]{bayati20}
Mohsen Bayati, Nima Hamidi, Ramesh Johari, and Khashayar Khosravi.
\newblock Unreasonable effectiveness of greedy algorithms in multi-armed bandit
  with many arms.
\newblock In H.~Larochelle, M.~Ranzato, R.~Hadsell, M.F. Balcan, and H.~Lin,
  editors, \emph{Advances in Neural Information Processing Systems}, volume~33,
  pages 1713--1723. Curran Associates, Inc., 2020.
\newblock URL \url{https://arxiv.org/pdf/2002.10121}.

\bibitem[Berry et~al.(1997)Berry, Chen, Zame, Heath, and Shepp]{berry97}
Donald~A. Berry, Robert~W. Chen, Alan Zame, David~C. Heath, and Larry~A. Shepp.
\newblock {Bandit problems with infinitely many arms}.
\newblock \emph{The Annals of Statistics}, 25\penalty0 (5):\penalty0 2103 --
  2116, 1997.
\newblock \doi{10.1214/aos/1069362389}.
\newblock URL \url{https://doi.org/10.1214/aos/1069362389}.

\bibitem[Besbes et~al.(2019)Besbes, Gur, and Zeevi]{besbes2014}
Omar Besbes, Yonatan Gur, and Assaf Zeevi.
\newblock Optimal exploration-exploitation in a multi-armed-bandit problem with
  non-stationary rewards.
\newblock \emph{Stochastic Systems}, 9\penalty0 (4):\penalty0 319--337, 2019.
\newblock URL
  \url{https://pubsonline.informs.org/doi/epdf/10.1287/stsy.2019.0033}.

\bibitem[Bonald and Proutiere(2013)]{bonald13}
Thomas Bonald and Alexandre Proutiere.
\newblock Two-target algorithms for infinite-armed bandits with bernoulli
  rewards.
\newblock In C.J. Burges, L.~Bottou, M.~Welling, Z.~Ghahramani, and K.Q.
  Weinberger, editors, \emph{Advances in Neural Information Processing
  Systems}, volume~26. Curran Associates, Inc., 2013.
\newblock URL
  \url{https://proceedings.neurips.cc/paper_files/paper/2013/file/fc2c7c47b918d0c2d792a719dfb602ef-Paper.pdf}.

\bibitem[Bubeck and Cesa-Bianchi(2012)]{bubeck2012a}
S\'{e}bastien Bubeck and Nicol\'{o} Cesa-Bianchi.
\newblock Regret analysis of stochastic and nonstochastic multi-armed bandit
  problems.
\newblock \emph{Foundations and Trends in Machine Learning}, 5\penalty0 (1),
  2012.
\newblock URL \url{https://arxiv.org/pdf/1204.5721.pdf}.

\bibitem[Buening and Saha(2023)]{buening22}
Thomas~Kleine Buening and Aadirupa Saha.
\newblock Anaconda: An improved dynamic regret algorithm for adaptive
  non-stationary dueling bandits.
\newblock In Francisco Ruiz, Jennifer Dy, and Jan-Willem van~de Meent, editors,
  \emph{Proceedings of The 26th International Conference on Artificial
  Intelligence and Statistics}, volume 206 of \emph{Proceedings of Machine
  Learning Research}, pages 3854--3878. PMLR, 25--27 Apr 2023.
\newblock URL
  \url{https://proceedings.mlr.press/v206/kleine-buening23a/kleine-buening23a.pdf}.

\bibitem[Cao et~al.(2019)Cao, Wen, Kveton, and Xie]{cao2019}
Yang Cao, Zheng Wen, Branislav Kveton, and Yao Xie.
\newblock Nearly optimal adaptive procedure with change detection for
  piecewise-stationary bandit.
\newblock \emph{Proceedings of the 22nd International Conference on Artificial
  Intelligence and Statistics (AISTATS)}, 2019.
\newblock URL \url{https://proceedings.mlr.press/v89/cao19a/cao19a.pdf}.

\bibitem[Carpentier and Valko(2015)]{carpentier15}
Alexandra Carpentier and Michal Valko.
\newblock Simple regret for infinitely many armed bandits.
\newblock \emph{ICML}, 2015.
\newblock URL \url{https://arxiv.org/pdf/1505.04627}.

\bibitem[Chen et~al.(2019)Chen, Lee, Luo, and Wei]{chen2019}
Yifang Chen, Chung-Wei Lee, Haipeng Luo, and Chen-Yu Wei.
\newblock A new algorithm for non-stationary contextual bandits: efficient,
  optimal, and parameter-free.
\newblock In \emph{32nd Annual Conference on Learning Theory}, 2019.
\newblock URL \url{https://proceedings.mlr.press/v99/chen19b/chen19b.pdf}.

\bibitem[Cheung et~al.(2019)Cheung, Simchi-Levi, and Zhu]{cheung2019learning}
Wang~Chi Cheung, David Simchi-Levi, and Ruihao Zhu.
\newblock Learning to optimize under non-stationarity.
\newblock In \emph{The 22nd International Conference on Artificial Intelligence
  and Statistics}, pages 1079--1087. PMLR, 2019.
\newblock URL \url{https://proceedings.mlr.press/v89/cheung19b/cheung19b.pdf}.

\bibitem[Dann et~al.(2023)Dann, Wei, and Zimmert]{dann23}
Chris Dann, Chen-Yu Wei, and Julian Zimmert.
\newblock A blackbox approach to best of both worlds in bandits and beyond.
\newblock In Gergely Neu and Lorenzo Rosasco, editors, \emph{Proceedings of
  Thirty Sixth Conference on Learning Theory}, volume 195 of \emph{Proceedings
  of Machine Learning Research}, pages 5503--5570. PMLR, 12--15 Jul 2023.
\newblock URL \url{https://proceedings.mlr.press/v195/dann23a.html}.

\bibitem[Even-Dar et~al.(2006)Even-Dar, Mannor, and Mansour]{even-dar}
Eyal Even-Dar, Shie Mannor, and Yishay Mansour.
\newblock Action elimination and stopping conditions for the multi-armed bandit
  and reinforcement learning problems.
\newblock \emph{Journal of Machine Learning Research}, 2006.
\newblock URL
  \url{https://jmlr.csail.mit.edu/papers/volume7/evendar06a/evendar06a.pdf}.

\bibitem[Garivier and Moulines(2011)]{garivier2011}
Aur\'{e}lien Garivier and Eric Moulines.
\newblock On upper-confidence bound policies for switching bandit problems.
\newblock In \emph{Proceedings of the 22nd International Conference on
  Algorithmic Learning Theory}, pages 174--188. ALT 2011, Springer, 2011.
\newblock URL \url{https://hal.science/hal-00281392/document}.

\bibitem[Gupta et~al.(2019)Gupta, Koren, and Talwar]{gupta19}
Anupam Gupta, Tomer Koren, and Kunal Talwar.
\newblock Better algorithms for stochastic bandits with adversarial
  corruptions.
\newblock \emph{Proceedings of the International Conference on Computational
  Learning Theory (COLT)}, 2019.
\newblock URL \url{https://arxiv.org/pdf/1807.07623}.

\bibitem[Hong et~al.(2023)Hong, Li, and Tewari]{hong23}
Kihyuk Hong, Yuhang Li, and Ambuj Tewari.
\newblock An optimization-based algorithm for non-stationary kernel bandits
  without prior knowledge.
\newblock In Francisco Ruiz, Jennifer Dy, and Jan-Willem van~de Meent, editors,
  \emph{Proceedings of The 26th International Conference on Artificial
  Intelligence and Statistics}, volume 206 of \emph{Proceedings of Machine
  Learning Research}, pages 3048--3085. PMLR, 25--27 Apr 2023.
\newblock URL \url{https://proceedings.mlr.press/v206/hong23b.html}.

\bibitem[Ito(2021)]{ito21}
Shinji Ito.
\newblock Parameter-free multi-armed bandit algorithms with hybrid
  data-dependent regret bounds.
\newblock In Mikhail Belkin and Samory Kpotufe, editors, \emph{Proceedings of
  Thirty Fourth Conference on Learning Theory}, volume 134 of \emph{Proceedings
  of Machine Learning Research}, pages 2552--2583. PMLR, 15--19 Aug 2021.
\newblock URL \url{https://proceedings.mlr.press/v134/ito21a.html}.

\bibitem[Ito and Takemura(2023)]{ito23}
Shinji Ito and Kei Takemura.
\newblock Best-of-three-worlds linear bandit algorithm with variance-adaptive
  regret bounds.
\newblock In Gergely Neu and Lorenzo Rosasco, editors, \emph{Proceedings of
  Thirty Sixth Conference on Learning Theory}, volume 195 of \emph{Proceedings
  of Machine Learning Research}, pages 2653--2677. PMLR, 12--15 Jul 2023.
\newblock URL \url{https://proceedings.mlr.press/v195/ito23a.html}.

\bibitem[Iwazaki and Takeno(2024)]{iwazaki24}
Shogo Iwazaki and Shion Takeno.
\newblock Near-optimal algorithm for non-stationary kernelized bandits.
\newblock \emph{arXiv preprint arXiv:2410.16052}, 2024.
\newblock URL \url{https://arxiv.org/pdf/2410.16052}.

\bibitem[Jia et~al.(2023)Jia, Xie, Kallus, and Frazier]{jia2023}
S.~Jia, Qian Xie, Nathan Kallus, and P.~Frazier.
\newblock Smooth non-stationary bandits.
\newblock In \emph{International Conference on Machine Learning}, 2023.
\newblock URL \url{https://dl.acm.org/doi/10.5555/3618408.3619017}.

\bibitem[Kim and Tewari(2020)]{kim20}
Baekjin Kim and Ambuj Tewari.
\newblock Randomized exploration for non-stationary stochastic linear bandit.
\newblock \emph{Uncertainty in Artificial Intelligence}, 2020.
\newblock URL \url{https://arxiv.org/pdf/1912.05695}.

\bibitem[Kim et~al.(2022)Kim, Vojnovic, and Yun]{kim22}
Jung-Hun Kim, Milan Vojnovic, and Se-Young Yun.
\newblock Rotting infinitely many-armed bandits.
\newblock In Kamalika Chaudhuri, Stefanie Jegelka, Le~Song, Csaba Szepesvari,
  Gang Niu, and Sivan Sabato, editors, \emph{Proceedings of the 39th
  International Conference on Machine Learning}, volume 162 of
  \emph{Proceedings of Machine Learning Research}, pages 11229--11254. PMLR,
  17--23 Jul 2022.
\newblock URL \url{https://proceedings.mlr.press/v162/kim22j.html}.

\bibitem[Kim et~al.(2024)Kim, Vojnovic, and Yun]{kim24}
Jung{-}hun Kim, Milan Vojnovic, and Se{-}Young Yun.
\newblock An adaptive approach for infinitely many-armed bandits under
  generalized rotting constraints.
\newblock In \emph{The Thirty-eighth Annual Conference on Neural Information
  Processing Systems}, 2024.
\newblock URL \url{https://arxiv.org/pdf/2404.14202}.

\bibitem[Kocsis and Szepesv\'{a}ri(2006)]{kocsis2006}
Levente Kocsis and Csaba Szepesv\'{a}ri.
\newblock Discounted ucb.
\newblock \emph{2nd PASCAL Challenges Workshop}, 2006.
\newblock URL \url{https://www.lri.fr/~sebag/Slides/Venice/Kocsis.pdf}.

\bibitem[Lai and Robbins(1985)]{lai}
T.L. Lai and Herbert Robbins.
\newblock Asymptotically efficient adaptive allocation rules.
\newblock \emph{Adv. in Appl. Math}, 1985.
\newblock URL
  \url{https://www.sciencedirect.com/science/article/pii/0196885885900028}.

\bibitem[Lattimore and Szepesv\'{a}ri(2020)]{lattimore}
Tor Lattimore and Csaba Szepesv\'{a}ri.
\newblock \emph{Bandit Algoritms}.
\newblock Cambridge University Press, 2020.
\newblock URL \url{https://tor-lattimore.com/downloads/book/book.pdf}.

\bibitem[Liu et~al.(2018)Liu, Lee, and Shroff]{liu2018}
Fang Liu, Joohyun Lee, and Ness Shroff.
\newblock A change-detection based framework for piecewise-stationary
  multi-armed bandit problem.
\newblock \emph{Proceedings of the AAAI Conference on Artificial Intelligence},
  2018.
\newblock URL \url{https://dl.acm.org/doi/pdf/10.5555/3504035.3504482}.

\bibitem[Lykouris et~al.(2018)Lykouris, Mirrokni, and Paes~Leme]{lykouris18}
Thodoris Lykouris, Vahab Mirrokni, and Renato Paes~Leme.
\newblock Stochastic bandits robust to adversarial corruptions.
\newblock In \emph{Proceedings of the 50th Annual ACM SIGACT Symposium on
  Theory of Computing}, STOC 2018, page 114–122, New York, NY, USA, 2018.
  Association for Computing Machinery.
\newblock ISBN 9781450355599.
\newblock \doi{10.1145/3188745.3188918}.
\newblock URL \url{https://arxiv.org/pdf/1803.09353}.

\bibitem[Manegueu et~al.(2021)Manegueu, Carpentier, and Yu]{manegueu2021}
Anne~Gael Manegueu, Alexandra Carpentier, and Yi~Yu.
\newblock Generalized non-stationary bandits.
\newblock \emph{arXiv preprint: arXiv:2102.00725}, 2021.
\newblock URL \url{https://arxiv.org/pdf/2102.00725.pdf}.

\bibitem[Mellor and Shapiro(2013)]{mellor13}
Joseph Mellor and Jonathan Shapiro.
\newblock Thompson sampling in switching environments with bayesian online
  change detection.
\newblock In Carlos~M. Carvalho and Pradeep Ravikumar, editors,
  \emph{Proceedings of the Sixteenth International Conference on Artificial
  Intelligence and Statistics}, volume~31 of \emph{Proceedings of Machine
  Learning Research}, pages 442--450, Scottsdale, Arizona, USA, 29 Apr--01 May
  2013. PMLR.
\newblock URL \url{https://proceedings.mlr.press/v31/mellor13a.html}.

\bibitem[Russac et~al.(2020)Russac, Capp{\'e}, and
  Garivier]{russac2020algorithms}
Yoan Russac, Olivier Capp{\'e}, and Aur{\'e}lien Garivier.
\newblock Algorithms for non-stationary generalized linear bandits.
\newblock \emph{arXiv preprint arXiv:2003.10113}, 2020.
\newblock URL \url{https://arxiv.org/pdf/2003.10113.pdf}.

\bibitem[Slivkins(2019)]{slivkinsbook}
Aleksandrs Slivkins.
\newblock Introduction to multi-armed bandits.
\newblock \emph{Foundations and Trends® in Machine Learning}, 12\penalty0
  (1-2):\penalty0 1--286, 2019.
\newblock ISSN 1935-8237.
\newblock \doi{10.1561/2200000068}.
\newblock URL \url{https://arxiv.org/pdf/1904.07272.pdf}.

\bibitem[Suk(2024)]{suk24}
Joe Suk.
\newblock Adaptive smooth non-stationary bandits.
\newblock \emph{arXiv preprint arXiv:2407.08654}, 2024.
\newblock URL \url{https://arxiv.org/pdf/2407.08654}.

\bibitem[Suk and Agarwal(2023)]{sukagarwal23}
Joe Suk and Arpit Agarwal.
\newblock When can we track significant preference shifts in dueling bandits?
\newblock \emph{Advances in Neural Information Processing Systems (NeurIPS)},
  2023.
\newblock URL \url{https://arxiv.org/pdf/2302.06595.pdf}.

\bibitem[Suk and Kpotufe(2022)]{suk22}
Joe Suk and Samory Kpotufe.
\newblock Tracking most significant arm switches in bandits.
\newblock \emph{Conference on Learning Theory (COLT)}, 2022.
\newblock URL \url{https://proceedings.mlr.press/v178/suk22a/suk22a.pdf}.

\bibitem[Suk and Kpotufe(2023)]{suk23}
Joe Suk and Samory Kpotufe.
\newblock Tracking most significant shifts in nonparametric contextual bandits.
\newblock \emph{Advances in Neural Information Processing Systems (NeurIPS)},
  2023.
\newblock URL \url{https://arxiv.org/pdf/2307.05341.pdf}.

\bibitem[Wang(2022)]{wang22}
Yining Wang.
\newblock Technical note—on adaptivity in nonstationary stochastic
  optimization with bandit feedback.
\newblock \emph{Operations Research}, 2022.
\newblock URL \url{https://arxiv.org/pdf/2210.05584.pdf}.

\bibitem[Wang et~al.(2008)Wang, Audibert, and Munos]{wang08}
Yizao Wang, Jean-yves Audibert, and R\'{e}mi Munos.
\newblock Algorithms for infinitely many-armed bandits.
\newblock In D.~Koller, D.~Schuurmans, Y.~Bengio, and L.~Bottou, editors,
  \emph{Advances in Neural Information Processing Systems}, volume~21. Curran
  Associates, Inc., 2008.
\newblock URL
  \url{https://proceedings.neurips.cc/paper_files/paper/2008/file/49ae49a23f67c759bf4fc791ba842aa2-Paper.pdf}.

\bibitem[Wei and Luo(2021)]{wei2021}
Chen-Yu Wei and Haipeng Luo.
\newblock Non-stationary reinforcement learning without prior knowledge: An
  optimal black-box approach.
\newblock \emph{Proceedings of the 32nd International Conference on Learning
  Theory}, 2021.
\newblock URL \url{https://proceedings.mlr.press/v134/wei21b/wei21b.pdf}.

\bibitem[Yu and Mannor(2009)]{yu2009}
Jia~Yuan Yu and Shie Mannor.
\newblock Piecewise-stationary bandit problems with side observations.
\newblock \emph{Proceedings of the 26th Annual international Conference on
  Machine Learning}, pages 1177--1184, 2009.
\newblock URL
  \url{http://www.machinelearning.org/archive/icml2009/papers/367.pdf}.

\bibitem[Zhao et~al.(2020)Zhao, Zhang, Jiang, and Zhou]{zhao2020}
Peng Zhao, Lijun Zhang, Yuan Jiang, and Zhi-Hua Zhou.
\newblock A simple approach for non-stationary linear bandits.
\newblock \emph{International Conference on Artificial Intelligence and
  Statistics (AISTATS)}, 2020.
\newblock URL \url{https://arxiv.org/pdf/2103.05324}.

\bibitem[Zimmert and Lattimore(2022)]{zimmert22}
Julian Zimmert and Tor Lattimore.
\newblock Return of the bias: Almost minimax optimal high probability bounds
  for adversarial linear bandits.
\newblock In Po-Ling Loh and Maxim Raginsky, editors, \emph{Proceedings of
  Thirty Fifth Conference on Learning Theory}, volume 178 of \emph{Proceedings
  of Machine Learning Research}, pages 3285--3312. PMLR, 02--05 Jul 2022.
\newblock URL \url{https://proceedings.mlr.press/v178/zimmert22b.html}.

\bibitem[Zimmert and Seldin(2019)]{zimmert19}
Julian Zimmert and Yevgeny Seldin.
\newblock An optimal algorithm for stochastic and adversarial bandits.
\newblock In Kamalika Chaudhuri and Masashi Sugiyama, editors,
  \emph{Proceedings of the Twenty-Second International Conference on Artificial
  Intelligence and Statistics}, volume~89 of \emph{Proceedings of Machine
  Learning Research}, pages 467--475. PMLR, 16--18 Apr 2019.
\newblock URL \url{https://arxiv.org/pdf/1807.07623}.

\end{thebibliography}

\newpage
\appendix

\section{Blackbox Algorithm Regret Analysis (Details for the Proof of \Cref{thm:blackbox})}\label{app:blackbox}

Here, we present the details of the proof of \Cref{thm:blackbox} following the outline of \Cref{sec:blackbox}.
Again, we focus on the case of $\beta \geq 1$ and discuss modifications of the argument required for $\beta < 1$ in \Cref{subsec:beta-small}.

\subsection{Preliminaries}\label{subsec:prelim}

Let $c_0,c_1,c_2,\ldots$ denote constants not depending on $T$, $\kappa_1$, or $\kappa_2$ (\Cref{ass:reservoir}).
In what follows, all logarithms will be base $2$.
We'll also assume WLOG that $\log(T) \geq \kappa_1^{-\frac{1}{\beta}} \vee \kappa_2$, as otherwise we can bound the regret by a constant only depending on $\kappa_1$ and $\kappa_2$.

Next, we establish a basic fact about the block structure of \Cref{alg:blackbox}.

\begin{fact}\label{fact:ep-dominate-log}
	Recall from \Cref{sec:blackbox} that $m_{\ell}$ is the index of the last block in the $\ell$-th episode.
	Then, for any episode $\ho{t_{\ell}}{t_{\ell+1}}$ terminating in a restart (via \Cref{line:cpd-test} of \Cref{alg:blackbox}), we have
	\begin{align}\label{eq:sm-small}
		S_{m_{\ell}} = \ceil{ 2^{m_{\ell} \cdot \frac{\beta}{\beta+1}} } \leq 2^{m_{\ell}},
	\end{align}
	and
	\begin{align}
		(t_{\ell+1} - t_{\ell}^{m_{\ell}})^{ \frac{1}{\beta+1}} &\geq \frac{C_1 }{2^{\frac{1}{\beta+1}}} \log^{3}(T) \label{eq:ep-large-enough-beta}
	\end{align}
\end{fact}

\begin{proof}
	By \Cref{line:cpd-test} of \Cref{alg:blackbox}, we must have
	\[
		2^{m_{\ell}} \geq t_{\ell+1} - t_{\ell}^{m_{\ell}} + 1 \geq \sum_{s = t_{\ell}^{m_{\ell}}}^{t_{\ell+1}-1} 1 - Y_s(a_s) \geq C_1 \cdot |\mc{A}_{m_{\ell}}| \cdot \log^3(T). 
	\]
	Now, recall the definition of $S_m$ so that $|\mc{A}_{m_{\ell}}| := \ceil{ 2^{m \cdot \frac{\beta}{\beta+1}}  } \land 2^m$.
	We first note that we cannot have $|\mc{A}_{m_{\ell}}| = 2^{m_{\ell}}$ since the above display would then become $1 \geq C_1 \log^3(T)$.
	Thus
	\[
		t_{\ell+1} - t_{\ell}^{m_{\ell}} + 1 \geq C_1 \cdot 2^{m_{\ell} \cdot \frac{\beta}{\beta+1}} \cdot \log^3(T) \geq C_1  \cdot (t_{\ell+1} - t_{\ell}^{m_{\ell}} + 1)^{\frac{\beta}{\beta+1}} \cdot \log^3(T).
	\]
	Rearranging this becomes $(t_{\ell+1} - t_{\ell}^{m_{\ell}} + 1)^{\frac{1}{\beta+1}} \geq C_1 \log^3(T)$.
	Further, bounding $t_{\ell+1} - t_{\ell}^{m_{\ell}} + 1 \leq 2 \cdot (t_{\ell+1} - t_{\ell}^{m_{\ell}})$ finishes the proof.
\end{proof}


\subsection{Using Concentration to Bound Per-Block Regret (Proof of \Cref{lem:concentration-bounds})}\label{subsec:proof-concentration-bounds}

We first present Freedman's inequality which is used in \Cref{sec:blackbox} to bound the regret using the empirical regret bound of \Cref{line:cpd-test} of \Cref{alg:blackbox}.

\begin{lemma}[Strengthened Freedman's Inequality, Theorem 9 \citep{zimmert22}]\label{lem:freedman}
	Let $X_1,X_2,\ldots,X_T$ be a martingale difference sequence with respect to a filtration $\mc{F}_1 \subseteq \mc{F}_2 \subseteq \cdots \subseteq \mc{F}_T$ such that $\mb{E}[X_t \mid \mc{F}_t ] = 0$ and assume $\mb{E}[ | X_t | \mid \mc{F}_t] < \infty$ a.s.. Then, with probability at least $1-\delta$,
	\[
		\sum_{t=1}^T X_t \leq 3 \sqrt{ V_t \log\left( \frac{2 \max\{ U_t, \sqrt{V_T}\}}{\delta} \right) } + 2 U_T \log \left( \frac{2 \max\{ U_T, \sqrt{V_T}\}}{\delta} \right),
	\]
	where $V_T = \sum_{t=1}^T \mb{E}[ X_t^2 \mid \mc{F}_t]$, and $U_T = \max \{ 1 ,\max_{s \in [T]} X_s \}$.
\end{lemma}

We next use this to relate the per-episode regret to the empirical bounds of \Cref{line:cpd-test}.

\begin{lemma}\label{lem:concentration-bounds}
	Let $\mc{E}_1$ be the event that (a) for all blocks $\ho{t_{\ell}^m}{t_{\ell}^{m+1}}$,
	\begin{equation}\label{eq:block-bound}
		\sum_{s=t_{\ell}^m}^{t_{\ell}^{m+1}-1} \delta_s(a_s) < 3 C_1 \cdot |\mc{A}_m| \cdot \log^3(T).
	\end{equation}
	and (b) for the last block $\ho{t_{\ell}^{m_{\ell}}}{t_{\ell+1}}$ of episodes $\ho{t_{\ell}}{t_{\ell+1}}$ concluding in a restart, we have:
	\begin{equation}\label{eq:large-regret-last-block}
		\sum_{s=t_{\ell}^{m_{\ell}}}^{t_{\ell+1} - 1} \delta_s(a_s) \geq \frac{C_1}{2} \cdot |\mc{A}_{m_{\ell}}| \cdot \log^3(T).
	\end{equation}
	Then, $\mc{E}_1$ occurs with probability at least $1-1/T$.
\end{lemma}

\begin{proof}
First, fix an interval of rounds $[s_1,s_2]$.
Then, we note that
\[
	\sum_{s=s_1}^{s_2} 1 - Y_s(a_s) - \mb{E}[ 1 - Y_s(a_s) \mid \mc{H}_{s-1}],
\]
is a martingale difference sequence with respect to the natural filtration $\{\mc{H}_t\}_t$ of $\sigma$-algebras $\mc{H}_{t}$, generated by the observations and decisions (of both decision-maker and adversary) up to round $t-1$.
Now, since the choice of arm $a_t$ at round $t$ is fixed conditional on $\mc{H}_{t-1}$, we have $\mb{E}[ Y_t(a_t) \mid \mc{H}_{t-1}] = \mu_t(a_t)$.
We also have since $Y_t(a_t) \in [0,1]$:
\[
	\Var( 1 - Y_t(a_t) \mid \mc{H}_{t-1} ) \leq \mb{E}[ (1 - Y_t(a_t))^2 \mid \mc{H}_{t-1}] \leq \mb{E}[ 1 - Y_t(a_t) \mid \mc{H}_{t-1} ] = \delta_t(a_t).
\]
Then, by Freedman's inequality (\Cref{lem:freedman}) we have with probability at least $1 - 1/T^3$: for all choices of intervals $[s_1,s_2] \subseteq [T]$:
\begin{align}
	\left| \sum_{s=s_1}^{s_2} (1 - Y_s(a_s)) - (1 - \mu_s(a_s))  \right| &\leq 3 \sqrt{  \log\left(2 T^3 \right) \sum_{s=s_1}^{s_2} \delta_s(a_s) } + 2 \log\left( 2 T^3 \right) \nonumber \\
									     &\leq \half \sum_{s=s_1}^{s_2} \delta_s(a_s) + \frac{13}{2}\log(2T^3), \label{eq:good_event1}
\end{align}
where the second inequality is by AM-GM.
Going forward, suppose the above concentration holds for all subintervals of rounds $[s_1,s_2] \subseteq [T]$.

At the same time, on each $m$-th block, \Cref{line:cpd-test} of \Cref{alg:blackbox} gives us an empirical regret upper bound since the changepoint test is not triggered until possibly the last round of the block.
Specifically, the inequality in \Cref{line:cpd-test} of \Cref{alg:blackbox} must be reversed for the second-to-last round of the block or
\[
	\sum_{s=t_{\ell}^m}^{t_{\ell}^{m+1} -2} 1 - Y_s(a_s) < C_1 \cdot ( |\mc{A}_m| \vee 2^{m/2}) \cdot \log^3(T), 
\]
where by convention we let $t_{\ell}^{m_{\ell}+1} := t_{\ell+1}$.
Thus,
\begin{equation}\label{eq:empirical-upper-bound}
	\sum_{s=t_{\ell}^m}^{t_{\ell}^{m+1}-1} 1 - Y_s(a_s) < C_1 \cdot (|\mc{A}_m| \vee 2^{m/2}) \cdot \log^3(T) + 1.
\end{equation}
Combining the above with \Cref{eq:good_event1}, we conclude
\[
	\sum_{s=t_{\ell}^m}^{t_{\ell}^{m+1}-1} \delta_s(a_s) < 2 C_1 \cdot ( |\mc{A}_m| \vee 2^{m/2} ) \cdot \log^3(T) + 2 + 13\log(2T^3) \leq 3 C_1 \cdot (|\mc{A}_m| \vee 2^{m/2} ) \cdot \log^3(T).
\]
where the last inequality holds for $C_1$ large enough.

At the same time, for constant $C_1$ in \Cref{line:cpd-test} of \Cref{alg:blackbox} chosen large enough, we have that, for the last block $\ho{t_{\ell}^{m_{\ell}}}{t_{\ell+1}}$ of an episode concluding in a restart, we must have
\begin{align*}
	\sum_{s=t_{\ell}^{m_\ell}}^{t_{\ell+1}-1} 1 - Y_s(a_s) &\geq C_1 \cdot ( |\mc{A}_m| \vee 2^{m/2} ) \cdot \log^3(T) \implies \\
	\sum_{s=t_{\ell}^{m_{\ell}}}^{t_{\ell+1}-1} \delta_s(a_s) &\geq \frac{2C_1}{3} \cdot ( |\mc{A}_{m_{\ell}}| \vee 2^{m/2} ) \cdot \log^3(T) - \frac{13}{3} \log(2T^3) \geq \frac{C_1}{2} \cdot ( |\mc{A}_{m_{\ell}}| \vee 2^{m/2} ) \cdot \log^3(T),
\end{align*}
where the first inequality in the second line above comes from combining the first line with \Cref{eq:good_event1} and the last inequality holds for $C_1$ large enough.

Finally, we note that for $\beta \geq 1$, we have $|\mc{A}_{m}| = S_{m} \land 2^m = \ceil{  2^{m \cdot \frac{\beta}{\beta+1}} } \land 2^{m} \geq 2^{m/2} $.
Thus, $|\mc{A}_m| \vee 2^{m/2} = |\mc{A}_m|$ in all our above inequalities.
\end{proof}

\subsection{Summing Regret Over Blocks}\label{subsec:proof-regret-bd-ep-length}

Next, we sum the per-block regret bound of \Cref{lem:concentration-bounds} over blocks $m$ and episodes $\ell$ to obtain a total regret bound.

\begin{lemma}\label{lem:regret-bd-episode-length}
	Under event $\mc{E}_1$, we have
    	\[
		\sum_{t=1}^T \delta_t(a_t) \leq c_0  \log^3(T) \sum_{\ell\in[\hat{L}]} (t_{\ell+1}-t_\ell)^{\frac{\beta}{\beta+1}} .
	\]
\end{lemma}

We first decompose the regret along episodes and blocks contained therein:
		\[
			\sum_{t=1}^T 1 - \mu_t(a_t) = \sum_{\ell = 1}^{\hat{L}} \sum_{m\in [m_\ell]} \sum_{t = t_{\ell}^m}^{t_{\ell}^{m+1}-1} 1 - \mu_t(a_t).
		\]
		Then, on event $\mc{E}_1$, we have summing \Cref{eq:block-bound}:
        \begin{align*}
            \sum_{t=1}^T 1 - \mu_t(a_t)  &\leq
	    \sum_{\ell\in [\hat{L}]} \sum_{m\in[m_\ell]} 3 C_1 \cdot |\mc{A}_m| \cdot \log^3(T) \\
					 &\leq c_1 \sum_{\ell \in [\hat{L}]} \sum_{m \in [m_{\ell}]} 2^{m \cdot \frac{\beta}{\beta+1}} \log^3(T) \\
					 &\leq c_2 \sum_{\ell \in [\hat{L}]} (t_{\ell+1} - t_{\ell})^{\frac{\beta}{\beta+1}} \cdot \log^3(T) .
        \end{align*}
	where in the last line we sum the geometric series over $m$ and the fact that $2^{m_{\ell} \cdot \frac{\beta}{\beta+1}} \leq (t_{\ell+1} - t_{\ell})^{\frac{\beta}{\beta+1}} \land 2^{m_{\ell}}$.

\subsection{Showing there is Large Variation in Each Episode}\label{subsec:proof-variation-bound}

Next, following the proof outline of \Cref{sec:blackbox}, our goal is to show there is a minimal amount of variation in each episode.

\begin{lemma}\label{lem:variation-bound}
	$\forall \ell\in [\hat{L}-1]$, w.p. $\geq 1-4/T$:
    \[
	    \sum_{t=t_{\ell}}^{t_{\ell+1}-1} |\mu_{t}(a_{t-1}) - \mu_{t-1}(a_{t-1})| \geq \frac{\log^3(T)}{(t_{\ell+1} - t_{\ell})^{\frac{1}{\beta+1}}}.
	 \]
\end{lemma}

\begin{proof}
As outlined in the proof outline of \Cref{sec:blackbox}, in light of the regret lower bound \Cref{eq:large-regret-last-block} of \Cref{lem:concentration-bounds}, we consider two different cases which we recall below:
	\begin{align}
		\sum_{t=t_\ell^{m_\ell}}^{t_{\ell+1}-1} 1 - \mu_{t}(\hat{a}_{\ell,m_\ell}) &\geq \frac{C_1}{4} |\mathcal{A}_{m_\ell}| \log^3(T) \tag{A}\label{eq:case1app} \\
		\sum_{t=t_\ell^{m_\ell}}^{t_{\ell+1}-1} 1 - \mu_{t}(\hat{a}_{\ell,m_\ell}) &\leq \frac{C_1}{4} |\mathcal{A}_{m_\ell}| \log^3(T) \text{ and } \sum_{t=t_{\ell}^{m_\ell}}^{t_{\ell+1}-1} \mu_t(\hat{a}_{\ell,m_\ell}) - \mu_t(a_t) \geq \frac{C_1}{4} |\mathcal{A}_{m_\ell}| \log^3(T) \tag{B}\label{eq:case2app}
	\end{align}
	Now, on event $\mc{E}_1$, due to \Cref{eq:large-regret-last-block}, one of \Cref{eq:case1app} or \Cref{eq:case2app} must hold.
	Our goal is to show that in either case, large variation must have elapsed over the episode.

	\paragraph*{$\bullet$ Best Initial Arm has Large Dynamic Regret.}
	We first consider the former case \Cref{eq:case1app}.
	First, we establish a lemma asserting the best initially subsampled arm has small initial gap with high probability.

\begin{lemma}{(Proof in \Cref{subsec:proof-best-initial-arm})}\label{lem:best-initial-arm}
	Recall that $\mu_0(a)$ is the initial mean reward of arm $a$.
	Let $\hat{a}_{\ell,m_{\ell}}$ denote the best initial arm among the arms sampled in the last block of episode $\ho{t_{\ell}}{t_{\ell+1}}$ or $\hat{a}_{\ell,m_{\ell}} := \argmax_{a \in \mc{A}_{m_{\ell}}} \mu_0(a)$.
	Let $\mc{E}_3$ be the event that
	\[
		\forall \ell \in [\hat{L}] : 1 - \mu_0(\hat{a}_{\ell, m_{\ell}}) \leq \frac{ \log^3(T)} { 2^{ m_{\ell} \cdot \frac{1}{\beta+1}}  }.
	\]
	Then, $\mc{E}_3$ occurs with probability at least $1-1/T$.
\end{lemma}

Now, from \eqref{eq:case1app} (with large enough $C_1 > 16$), we also know that there exists a round $t' \in \ho{t_{\ell}^{m_\ell}}{t_{\ell+1}}$ such that
	\[
		1- \mu_{t'}(\hat{a}_{\ell,m_\ell}) > \frac{C_1 \cdot |\mc{A}_{m_{\ell}}| \cdot \log^3(T)}{4 (t_{\ell+1} - t_{\ell}^{m_{\ell}} ) } \geq \frac{16  \cdot 2^{m_{\ell} \cdot \frac{\beta}{\beta+1}} \cdot \log^3(T)}{ 4 \cdot 2  \cdot 2^{m_{\ell}}} \geq \frac{2  \log^3(T)}{ 2^{m_{\ell} \frac{1}{\beta+1}}}.
	\]
	This means the mean reward of arm $\hat{a}_{\ell,m_\ell}$ must have moved by amount at least $ 2^{- m_{\ell} \cdot \frac{1}{\beta+1}} \log^3(T)$ over the course of the block, implying
	\[
		\mu_0(\hat{a}_{\ell,m_\ell})-\mu_{t'}(\hat{a}_{\ell,m_\ell})\ge \frac{\log^3(T)}{ 2^{ m_{\ell} \cdot \frac{1}{\beta+1} } } \geq \frac{\log^3(T)}{ (t_{\ell+1} - t_{\ell})^{\frac{1}{\beta+1}} }.
	\]
	Since the adversary can only modify the rewards of an arm $a_t$ after the round $t$ it is played, the movement must have occurred on rounds where $\hat{a}_{\ell,m_\ell}$ was chosen by the agent.
	Thus,
	\[
		\sum_{t=t_{\ell}}^{t_{\ell+1}-1} |\mu_{t}(a_{t-1}) - \mu_{t-1}(a_{t-1})|\ge \sum_{t=t_{\ell}}^{t_{\ell+1}-1} |\mu_{t}(\hat{a}_{\ell,m_\ell}) - \mu_{t-1}(\hat{a}_{\ell,m_\ell})|\ge \mu_0(\hat{a}_{\ell,m_\ell})-\mu_{t'}(\hat{a}_{\ell,m_\ell}) \geq \frac{\log^3(T)}{ (t_{\ell+1} - t_{\ell})^{\frac{1}{\beta+1}} }.
	\]

	\paragraph*{$\bullet$ Regret of Base Algorithm is Large.}
	Now, we consider the other case \Cref{eq:case2app}, where
	the regret of the subsampled bandit environment over the rounds $\ho{t_{\ell}^{m_\ell}}{t_{\ell+1}}$ is large, while the dynamic regret of the best initial arm is small.
	In this case, we want to show a contradiction: that if the total variation over the episode is small, then the finite MAB environment experienced by the base algorithm must be mildly corrupt (\Cref{defn:mild}) which means \Cref{assumption:base} can be used to bound the regret of the last block.
	This will yield a contradiction since by virtue of \Cref{line:cpd-test} being triggered, we know the regret of the last block be large.

	Suppose for contradiction that
	\begin{equation}\label{eq:small-variation-contradiction}
		\sum_{t=t_{\ell}}^{t_{\ell+1}-1} |\mu_t(a_{t-1}) - \mu_{t-1}(a_{t-1})| < \frac{\log^3(T)}{(t_{\ell+1} - t_{\ell})^{\frac{1}{\beta+1}}}.
	\end{equation}
	Then, since our adaptive adversary only changes the rewards of arms on the rounds they are played, we have the finite-armed MAB environment experienced by the base is $\alpha$-mildly corrupt (\Cref{defn:mild}) with respect to reference reward profile $\{\mu(a)\}_{a \in \mc{A}_{m_{\ell}}}$ and $\alpha := \frac{\log^3(T)}{(t_{\ell+1} - t_{\ell}^{m_{\ell}})^{\frac{1}{\beta+1}}}$.

	This means we can employ \Cref{assumption:base} to bound the regret.
	In particular, our end goal here is to show that
	\begin{equation}\label{eq:blackbox-regret-small}
		\sum_{t=t_{\ell}^m}^{t_{\ell+1}-1} \mu_t(\hat{a}_{\ell,m_{\ell}}) - \mu_t(a_t) < \frac{C_1}{4} \cdot |\mc{A}_{m_{\ell}}| \cdot \log^3(T) ,
	\end{equation}
	which will contradict \Cref{eq:case2app} and imply \Cref{eq:small-variation-contradiction} is true.
	First, by \Cref{assumption:base} with $\alpha = \frac{\log^3(T)}{(t_{\ell+1} - t_{\ell}^{m_{\ell}})^{\frac{1}{\beta+1}}}$, we have
	\[
		\sum_{t=t_{\ell}^{m_{\ell}}}^{t_{\ell+1}-1} \mu_t(\hat{a}_{\ell,m_{\ell}}) - \mu_t(a_t) \leq C_0 \left( \sum_{i=2}^{|\mc{A}_{m_{\ell}}|} \frac{\log(T)}{\Delta_{(i)}} \cdot \pmb{1}\left\{ \frac{\Delta_{(i)}}{4} \geq \frac{\log^3(T)}{(t_{\ell+1} - t_{\ell}^{m_{\ell}})^{\frac{1}{\beta+1}}} \right\} + (t_{\ell+1} - t_{\ell}^{m_{\ell}} )^{\frac{\beta}{\beta+1}} \cdot \log^3(T) \right).
	\]
	To ease notation, we now parametrize the arms in $\mc{A}_{m_{\ell}} \bs \{ \hat{a}_{\ell,m_{\ell}}\}$ via $\{2,\ldots,|\mc{A}_{m_{\ell}}|\}$ and let $\Delta_i$ be the initial gap of arm $i$ to $\hat{a}_{\ell,m_{\ell}}$:
	\[
		\Delta_i := \mu_0(\hat{a}_{\ell,m_{\ell}}) - \mu_0(i),
	\]
	Now, we will partition the values of $\Delta_{i}$ based on a dyadic grid.
	Let
	\[
		N_{j,\ell} := \sum_{i=2}^{|\mc{A}_{m_{\ell}}|} \pmb{1}\{ \Delta_{i} \in \ho{2^{-(j+1)} - \delta_0(\hat{a}_{\ell,m_{\ell}})}{2^{-j} - \delta_0(\hat{a}_{\ell,m_{\ell}})} \},
	\]
	where $j$ ranges from $0$ to
	\[
		J := \ceil{  1 \vee \frac{1}{\beta} \log \left(  \frac{ (t_{\ell+1} - t_{\ell}^{m_{\ell}})^{\frac{\beta}{\beta+1}} }{42^{\beta} \log^{3\beta}(T)} \right) }
	\]
	Next, the following lemma bounds $N_{j,\ell}$ in high probability using Freedman's inequality (\Cref{lem:freedman}).

	\begin{lemma}{(Proof in \Cref{subsec:proof-bin-count})}\label{lem:bin-count-bound}
		Let $\mc{E}_4$ be the event that the following hold:
		\begin{align}
			\forall j \in \{0,\ldots,J\}, \ell \in \{1,\ldots,\hat{L} - 1\}: N_{j,\ell} \leq \frac{3 \kappa_2}{2} |\mc{A}_{m_{\ell}}| \cdot 2^{-j \beta} + \frac{13}{2} \log(2T^3) \label{eq:bound-bin-1} \\
			\sum_{i=2}^{|\mc{A}_{m_{\ell}}|} \pmb{1}\left\{ 4 \left( \frac{\log(T)}{t_{\ell+1} - t_{\ell}^{m_{\ell}}} \right)^{\frac{1}{\beta+1}} \leq \Delta_i < 2^{-(J+1)} \right\} \leq \frac{3\kappa_2}{2} |\mc{A}_{m_{\ell}}| \cdot 2^{-(J+1) \beta} + \frac{13}{2} \log(2T^3) \label{eq:bound-bin-2}
		\end{align}
		Then, $\mc{E}_4$ occurs with probability at least $1 - 1/T$.
	\end{lemma}

	The next lemma relates the cutoff $\Delta_i \geq 4 \alpha$ to the quantity $2^{-(J+2)}$ showing that every gap $\Delta_i$ such that $\Delta_i \geq 4 \alpha$ lies in one of the bins of the dyadic grid.

	\begin{lemma}{(Proof in \Cref{subsec:proof-bigger})}\label{lem:bigger}
		We have
		\[
			2^{-(J+2)} - \delta_0(\hat{a}_{\ell,m_{\ell}}) > 4 \alpha. 
		\]
	\end{lemma}

	Now, \Cref{lem:bigger} implies $2^{-(J+1)} - \delta_0(\hat{a}_{\ell,m_{\ell}}) > 4 \alpha$.
	Thus, every gap $\Delta_{i}$ such that $\Delta_{i}/4 \geq  \frac{\log^3(T)}{(t_{\ell+1} - t_{\ell}^{m_{\ell}})^{\frac{1}{\beta+1}}} $ lies in some interval $[2^{-(j+1)} - \delta_0(\hat{a}_{\ell,m_{\ell}}) , 2^{-j}  - \delta_0(\hat{a}_{\ell,m_{\ell}}))$ for some $j \in \{0,\ldots,J\}$ or lies in the interval $\ho{4\alpha}{2^{-(J+1)} - \delta_0(\hat{a}_{\ell,m_{\ell}})}$.
	Thus, we bound
	\begin{align}
		\sum_{i=2}^{|\mc{A}_{m_\ell}|} \frac{\log(T)}{\Delta_{i}} \cdot \pmb{1} \left\{ \frac{\Delta_{i}}{4} \geq \alpha \right\} &\leq \sum_{j=0}^{J} N_{j,\ell} \cdot (2^{-(j+1)} - \delta_0(\hat{a}_{\ell,m_{\ell}}))^{-1} \cdot \log(T) \nonumber\\
		&\qquad + \log(T) \cdot \alpha^{-1} \sum_{i=2}^{|\mc{A}_{m_{\ell}}|} \pmb{1}\{ \alpha \leq \Delta_i < 2^{-(J+1)} - \delta_0(\hat{a}_{\ell,m_{\ell}}) \} \label{eq:gap-indicator-regret-bd}
	\end{align}

	Now, by \Cref{lem:bigger}, we have for all $j \in \{0,\ldots,J\}$:
	\[
		2^{-(j+1)} - \delta_0(\hat{a}_{\ell,m_{\ell}}) > 2^{-(j+2)}.
	\]
	Combining the above with \Cref{lem:bin-count-bound}, we have \Cref{eq:gap-indicator-regret-bd} is order
	\[
		\log(T) \cdot \alpha^{-1} \left( \kappa_2 |\mc{A}_{m_{\ell}}| \cdot 2^{-(J+1) \beta} + \log(T) \right) + \log(T) \sum_{j=0}^J  \kappa_2 \cdot |\mc{A}_{m_{\ell}}| \cdot 2^{j (1 - \beta) + 2} + 2^{(j+1)} \log(T) .
	\]
	We have
	\begin{align*}
		2^{-(J+1)\beta} \leq 2^{-\beta} \left( \frac{1}{2^{\beta}} \land \left( \frac{42^{\beta} \log^{3\beta}(T)}{ (t_{\ell+1} - t_{\ell}^{m_{\ell}})^{\frac{\beta}{\beta+1}} } \right) \right) \leq \frac{42^{\beta} \log^{3\beta}(T)}{2^{\beta} (t_{\ell+1} - t_{\ell}^{m_{\ell}})^{\frac{\beta}{\beta+1}} }.
	\end{align*}
	Then, 
	\begin{align*}
		\kappa_2 \cdot \log(T) \cdot \alpha^{-1} \cdot 2^{-(J+1)\beta} &\leq c_3 \kappa_2 \cdot  (t_{\ell+1} - t_{\ell}^{m_{\ell}})^{\frac{1-\beta}{\beta+1}} \log^{3\beta - 3}(T) \\
							      &\leq c_4 \kappa_2 \log^{3(1 - \beta)}(T) \cdot \log^{3\beta - 3}(T) \\
							      &= c_5 \log(T),
	\end{align*}
	where the second inequality follows from \Cref{eq:ep-large-enough-beta} of \Cref{fact:ep-dominate-log} and $\beta \geq 1$, and the third inequality follows from $\log(T) \geq \kappa_2$ by choosing $T$ sufficiently large.
	Thus,
	\[
		\log(T) \cdot \alpha^{-1} \left( \kappa_2 |\mc{A}_{m_{\ell}}| \cdot 2^{-(J+1) \beta} + \log(T) \right)  \leq c_6 |\mc{A}_{m_{\ell}}| \log(T).
	\]
	Next,
	\begin{align*}
		\kappa_2 \log(T) \sum_{j=0}^J 2^{j(1-\beta)+2} &\leq 4 (J+1) \cdot 2^{J(1-\beta)} \log^2(T) \\
						      &\leq 8\log^2(T) \left( 2^{1-\beta} \vee \frac{(t_{\ell+1} - t_{\ell}^{m_{\ell}})^{\frac{1 - \beta}{\beta+1}}}{ 42^{1-\beta} \log^{2(1-\beta)}(T) } \right)  \left( 1 \vee \log\left( \frac{(t_{\ell+1} - t_{\ell}^{m_{\ell}})^{\frac{1}{\beta+1}} }{ 42^{\beta} \log^3(T) } \right) \right) \\
						      &\leq 8 \log^2(T) \left( 1 + \log(T) \right) \\
						      &\leq c_7 \log^3(T),
	\end{align*}
	where the first inequality follows from $\kappa_2 \leq \log(T)$, and the third inequality follows from $\beta \geq 1$ and \Cref{eq:ep-large-enough-beta} of \Cref{fact:ep-dominate-log}.

Thus, \Cref{eq:gap-indicator-regret-bd} is at most order $|\mc{A}_{m_{\ell}}| \log^3(T)$.
Thus, choosing $C_1$ large enough in \Cref{line:cpd-test} of \Cref{alg:blackbox}, we have that \Cref{eq:blackbox-regret-small} holds.
Following earlier discussion, this means the subsampled bandit environment over the last block $[t_\ell^m,t_{\ell+1})$ is \textit{not} mildly corrupt (\Cref{defn:mild}).
Then, for any $\mu(\cdot): \mc{A}_{m_\ell}\rightarrow [0,1]$, there always exists $t'\in[t_\ell^m,t_{\ell+1})$ and $a'\in \mc{A}_{m_\ell}$ such that
\[
	|\mu_{t'}(a')-\mu(a')|> \frac{\log^3(T)}{(t_{\ell+1}-t_\ell^{m_\ell})^{\frac{1}{\beta+1}}}.
\]

This means that since the adversary can only modify the rewards of arms $a_t$ after the round $t$, we have
\begin{align*}
\sum_{t=t_{\ell}}^{t_{\ell+1}-1} |\mu_{t}(a_{t-1}) - \mu_{t-1}(a_{t-1})|&\ge
	\sum_{t=t_{\ell}^{m_\ell}}^{t_{\ell+1}-1} |\mu_{t}(a_{t-1}) - \mu_{t-1}(a_{t-1})|\cr &\ge   \sum_{t=t_{\ell}^{m_\ell}}^{t_{\ell+1}-1}|\mu_{t}(a') - \mu_{t-1}(a')|\cr &\ge |\mu_{t'}(a') - \mu_{0}(a')| \cr & \geq \frac{\log^3(T)}{ (t_{\ell+1} - t_{\ell}^{m_\ell})^{\frac{1}{\beta+1}}}\cr & \geq \frac{\log^3(T)}{ (t_{\ell+1} - t_{\ell})^{\frac{1}{\beta+1}} }.
\end{align*}

\end{proof}

\subsection{Relating Episodes to Non-Stationarity Measures}\label{subsec:proof-sum}

Now, we show how to derive
\[
	\sum_{\ell=1}^{\hat{L}} (t_{\ell+1} - t_{\ell})^{\frac{\beta}{\beta+1}} \log^3(T) \leq \left( (L+1)^{\frac{1}{\beta+1}} \cdot T^{\frac{\beta}{\beta+1}}  \land (V^{\frac{1}{\beta+2}} T^{\frac{\beta+1}{\beta+2}} + T^{\frac{\beta}{\beta+1}} ) \right) \cdot \log^3(T),
\]
from the fact that, as we have shown by \Cref{lem:variation-bound}, for all episodes $\ell\in [\hat{L}-1]$
\[
	\sum_{t=t_{\ell}}^{t_{\ell+1}-1} |\mu_{t}(a_{t-1}) - \mu_{t-1}(a_{t-1})|\ge \frac{\log^3(T)}{ (t_{\ell+1} - t_{\ell})^{\frac{1}{\beta+1}}}.
\]
We first show the bound in terms of total variation $V$.
Let $V_{[s_1,s_2)} := \sum_{t=s_1}^{s_2-1}|\mu_t(a_{t-1})-\mu_{t-1}(a_{t-1})|$. Then, by using Hölder's inequality:
\begin{align*}
	\sum_{\ell=1}^{\hat{L}} (t_{\ell+1}-t_\ell)^{\frac{\beta}{\beta+1}} \log^3(T) &\le T^{\frac{\beta}{\beta+1}} \log^3(T) + \left(\sum_{\ell=1}^{\hat{L}-1} \frac{\log^3(T)}{(t_{\ell+1}-t_\ell)^{\frac{1}{\beta+1}}} \right)^{\frac{1}{\beta+2}} \left( \sum_{\ell=1}^{\hat{L}-1} (t_{\ell+1}-t_\ell) \log^{3-\frac{3}{\beta+2}}(T) \right)^{\frac{\beta+1}{\beta+2}} \\
										      &\le  T^{\frac{\beta}{\beta+1}} \log^3(T) + \left(\sum_{\ell=1}^{\hat{L}-1} V_{[t_\ell, t_{\ell+1})} \right)^{\frac{1}{\beta+2}} T^{\frac{\beta+1}{\beta+2}} \cdot \log^{\frac{3(\beta+1)^2}{(\beta+2)^2}}(T) \\
										      &\le T^{\frac{\beta}{\beta+1}} \log^3(T) + V^{\frac{1}{\beta+2}} T^{\frac{\beta+1}{\beta+2}} \cdot \log^{\frac{3(\beta+1)^2}{(\beta+2)^2}}(T)
\end{align*}


To show the bound in terms of $L$, we use Jensen's inequality on the function $x \mapsto x^{\frac{\beta}{\beta+1}}$ combined with the fact that the number of episodes $\hat{L} \leq L+1$ by virtue of \Cref{lem:variation-bound}.

\subsection{A Sketch of Modifications Required for Proving \Cref{thm:blackbox} for $\beta < 1$}\label{subsec:beta-small}

Next, we describe how to show a suboptimal regret bound of order $\sum_{\ell=1}^{\hat{L}} \sqrt{ t_{\ell+1} - t_{\ell}}$ in the setting of $\beta < 1$.
From here, using the steps of \Cref{subsec:proof-sum} with $\beta=1$, it is straightforward to show a regret bound of $\sqrt{(L+1) T} \land (V^{1/3} T^{2/3} + \sqrt{T})$.
For the sake of redundancy, we only give here a sketch of the modifications to the argument required.

We'll first describe at a high level the difficulty of showing a $\sum_{\ell=1}^{\hat{L}} (t_{\ell+1} - t_{\ell})^{\frac{\beta}{\beta+1}}$ regret bound using the same arguments of the previous sections for $\beta < 1$.
In fact, the only place where $\beta \geq 1$ was used is in the final step where we bound \Cref{eq:gap-indicator-regret-bd}.
In particular, when bounding the sum $\sum_{j=0}^J 2^{j(1-\beta)+2} \leq 4 (J+1) \cdot 2^{J(1-\beta)}$, for $\beta \geq 1$, $2^{J(1-\beta)}$ is a constant.
However, for $\beta < 1$, we may incur an extra $(t_{\ell+1} - t_{\ell}^{m_{\ell}})^{\frac{1-\beta}{\beta+1}}$ term in the final regret bound due to this term.
Thus, this argument would only yield a suboptimal regret bound.
Interestingly, \citet[cf. E.1.2]{bayati20} and \citet{kim24} also face this difficulty in bounding similar terms which leads to a suboptimal rate.

However, we can still attain a $\sqrt{t_{\ell+1} - t_{\ell}}$ per-episode regret bound using an altered subsampling rate $S_m := \ceil{  2^{m\cdot \beta/2}} \land 2^m$ which effectively ``targets'' a $\sqrt{t_{\ell+1} - t_{\ell}}$ regret rate.

Going into more detail, a first key fact, as observed in \citet{bayati20}, is that subsampling $S_m$ arms ensures an arm with gap $\tilde{O}((t_{\ell+1} - t_{\ell})^{-1/2})$.
This can be seen analogously to the proof of \Cref{lem:best-initial-arm} in \Cref{subsec:proof-best-initial-arm} where letting $\beta = 1$ and $|\mc{A}_{m_{\ell}}| \gtrsim 2^{ m_{\ell} \cdot \beta/2}$ in \Cref{subsec:proof-best-initial-arm} establishes that the best initial arm has gap at most $\log^3(T) \cdot 2^{- m_{\ell} \cdot \beta/2}$.

Then, the key fact to show will be an analogue of \Cref{lem:variation-bound}: with probability at least $1-4/T$, for all $\ell \in [ \hat{L} - 1]$:
\begin{equation}\label{eq:variation-ep-beta-small}
	\sum_{t=t_{\ell}}^{t_{\ell+1}-1} |\mu_t(a_{t-1}) - \mu_{t-1}(a_{t-1})| \geq \frac{\log^3(T)}{ \sqrt{ t_{\ell+1} - t_{\ell} } }.
\end{equation}
To show this, we'll essentially repeat the arguments of \Cref{subsec:proof-variation-bound} except specializing ``$\beta=1$'' to target a bound scaling like $\sqrt{t_{\ell+1} - t_{\ell}}$.
Note that using the scaling $|\mc{A}_m| \vee 2^{m/2}$ in the regret threshold for the changepoint detection test (\Cref{line:cpd-test} of \Cref{alg:blackbox}) is crucial here as $|\mc{A}_m| \propto 2^{m \cdot \beta/2} \ll 2^{m/2}$ in the case of $\beta < 1$.

First, if the dynamic regret of the best initial arm $\hat{a}_{\ell,m_{\ell}}$ is larger than $\frac{C_1}{4} 2^{m_{\ell} /2} \cdot \log^3(T)$, then using the previously established fact that $\delta_0(\hat{a}_{\ell,m_{\ell}}) \leq \log^3(T) \cdot 2^{-m_{\ell} \cdot \beta/2}$, we have that \Cref{eq:variation-ep-beta-small} holds.

Next, following the argument structure of \Cref{subsec:proof-variation-bound}, suppose $\sum_{t = t_{\ell}^{m_{\ell}}}^{t_{\ell+1}-1} \delta_t(\hat{a}_{\ell,m_{\ell}}) \leq \frac{C_1}{4} 2^{m_{\ell}/2} \log^3(T)$ but $\sum_{t=t_{\ell}^{m_{\ell}}}^{t_{\ell+1}-1} \mu_t(\hat{a}_{\ell,m_{\ell}}) - \mu_t(a_t) \geq \frac{C_1}{4} 2^{m_{\ell}/2} \log^3(T)$.
Then, we invoke \Cref{assumption:base} with $\alpha = \log^3(T) \cdot (t_{\ell+1} - t_{\ell}^{m_{\ell}})^{-1/2}$ and use the same dyadic gridding argument with
\[
	J := \ceil{ 1 \vee \log\left( \frac{(t_{\ell+1} - t_{\ell}^{m_{\ell}})^{1/2}}{42 \log^3(T)} \right) }.
\]
Then, observe the bounds
\begin{align*}
	\alpha^{-1} \cdot 2^{-(J+1)\beta} &\lesssim \frac{(t_{\ell+1} - t_{\ell}^{m_{\ell}})^{1/2}}{\log^3(T)} \cdot \frac{\log^3(T)}{(t_{\ell+1} - t_{\ell}^{m_{\ell}})^{1/2}} \lesssim O(1) \\
	J \cdot 2^{J(1-\beta)} \cdot |\mc{A}_{m_{\ell}}| &\lesssim 2^{m_{\ell} \cdot \beta/2} \cdot (t_{\ell+1} - t_{\ell}^{m_{\ell}})^{(1-\beta)/2} \leq 2^{m_{\ell}/2}.
\end{align*}
Thus, we can show $\sum_{t=t_{\ell}^{m_{\ell}}}^{t_{\ell+1}} \mu_t(\hat{a}_{\ell,m_{\ell}}) - \mu_t(a_t) < \frac{C_1}{4} 2^{m_{\ell}/2} \log^3(T)$ if \Cref{eq:variation-ep-beta-small} does not hold, which contradicts our earlier supposition.
This means \Cref{eq:variation-ep-beta-small} holds.

\subsection{Proof of \Cref{lem:best-initial-arm}}\label{subsec:proof-best-initial-arm}

Let $\mu_0(a)$ denote the initial mean reward of arm $a$.
By \Cref{ass:reservoir}, we have
\begin{align*}
	\mb{P}\left(\mu_0(\hat{a}_{\ell,m_\ell}) \leq 1 - \frac{  \log^3(T)}{ 2^{m_{\ell} \cdot \frac{1}{\beta+1}}} \right) &=      \mb{P}\left(\max_{a\in \mathcal{A}_{m_\ell}}\mu_0(a) \leq 1 - \frac{  \log^3(T)}{ 2^{m_{\ell} \cdot \frac{1}{\beta+1}}} \right)    \\
																	     &\leq \left( 1 - \mb{P}\left( \mu_0(a) > 1 - \frac{  \log^3(T)}{2^{m_{\ell} \cdot \frac{1}{\beta+1}}} \right) \right)^{|\mc{A}_{m_{\ell}}|} \\
																	     &\leq  \left( 1 - \kappa_1 \cdot \frac{ \log^{3\beta}(T)}{ 2^{m_{\ell} \cdot \frac{\beta}{\beta+1}}} \right)^{|\mc{A}_{{m_\ell}}|} \\
															&\leq \exp\left( - |\mc{A}_{{m_\ell}}| \cdot  \frac{\log(T)}{ 2^{m_{\ell} \cdot \frac{\beta}{\beta+1}}}  \right) \\
	&\leq \frac{1}{T^2},
\end{align*}
where the fourth line follows from assuming WLOG that $\log(T) \geq \kappa_1^{-\frac{1}{2\beta}}$ and $\beta\geq 1$, and
the last inequality from $|\mc{A}_{m_\ell}| \geq 2^{1 + m_{\ell} \cdot \frac{\beta}{\beta+1}}$.

\subsection{Proof of \Cref{lem:bin-count-bound}}\label{subsec:proof-bin-count}
By Freedman's inequality (\Cref{lem:freedman}), we have with probability at least $1/T^3$:
\[
	| N_{j,\ell} - \mb{E}[N_{j,\ell}] | \leq 3 \sqrt{ \mb{E}[ N_{j,\ell} ] \cdot \log(2T^3) } + 2\log(2T^3)
\]
Using AM-GM, this yields,
\[
	N_{j,\ell} \leq \frac{3\mb{E}[N_{j,\ell}]}{2} + \frac{13}{2} \log(2T^3).
\]
Finally,
\[
	\mb{P}( 2^{-(j+1)} - \delta_0(\hat{a}_{\ell,m_{\ell}}) \leq \Delta_i < 2^{-j} - \delta_0(\hat{a}_{\ell,m_{\ell}})) = \mb{P}( 2^{-(j+1)} \leq \delta_0(i) < 2^{-j}) \leq \mb{P}( 1 - 2^{-j} < \mu_0(i)) \leq \kappa_2 \cdot 2^{-j\beta}.
\]
Thus,
\[
	\mb{E}[ N_{j,\ell} ] \leq ( |\mc{A}_{m_{\ell}}| - 1) \cdot \kappa_2 \cdot 2^{-j \beta}.
\]
This shows \Cref{eq:bound-bin-1}.
\Cref{eq:bound-bin-2} is showed in a nearly identical manner. 

\subsection{Proof of \Cref{lem:bigger}}\label{subsec:proof-bigger}

	Recall $\alpha :=  \frac{\log^3(T)}{(t_{\ell+1} - t_{\ell}^{m_{\ell}})^{\frac{1}{\beta+1}}} $ and let $\delta_0 := \delta_0(\hat{a}_{\ell,m_{\ell}})$ to ease notation.
	First, suppose $J = 1$.
	Then, by \Cref{lem:best-initial-arm}
	\[
		2^{-3} - \delta_0 > 4 \alpha \impliedby \frac{1}{8} > \frac{\log^3(T)}{2^{m_{\ell} \cdot \frac{1}{\beta+1}}} + \frac{4 \log^{3}(T)}{ (t_{\ell+1} - t_{\ell}^{m_{\ell}})^{\frac{1}{\beta+1}}}.
	\]
	Now, since $\beta \geq 1$ and $2^{m_{\ell}} \geq t_{\ell+1} - t_{\ell}^{m_{\ell}}$, it suffices to show
	\[
		\frac{1}{8} > \frac{5 \log^3(T)}{ (t_{\ell+1} - t_{\ell}^{m_{\ell}})^{\frac{1}{\beta+1}}} \iff (t_{\ell+1} - t_{\ell}^{m_{\ell}})^{\frac{1}{\beta+1}} > 40 \log^3(T).
	\]
	However, this last inequality is true from \Cref{eq:ep-large-enough-beta} of \Cref{fact:ep-dominate-log} for sufficiently large $C_1$. 

Now, suppose $J > 1$.
	Rearranging the desired inequality we have
	\begin{align*}
		(\delta_0 + 4 \alpha)^{-1} > 2^{J+2} \impliedby
	\frac{1}{4 (\delta_0 + 4\alpha)} > 4 \vee 2 \cdot \frac{(t_{\ell+1} - t_{\ell}^{m_{\ell}})^{\frac{1}{\beta+1}}}{ 42 \log^3(T) }
	\end{align*}
	From \Cref{eq:ep-large-enough-beta} of \Cref{fact:ep-dominate-log}, we have for sufficiently large $C_1$:
	\[
		2 \cdot \frac{(t_{\ell+1} - t_{\ell}^{m_{\ell}})^{\frac{1}{\beta+1}}}{ 42 \log^3(T) } > 4.
	\]
	Thus, in light of \Cref{lem:best-initial-arm} and the definition of $\alpha$ we want to show
	\[
		\frac{ 21 \log^3(T)}{(t_{\ell+1} - t_{\ell}^{m_{\ell}})^{\frac{1}{\beta+1}}} > \frac{ 4 \log^3(T)}{(t_{\ell+1} - t_{\ell}^{m_{\ell}})^{\frac{1}{\beta+1}}} + \frac{16\log^3(T)}{(t_{\ell+1} - t_{\ell}^{m_{\ell}})^{\frac{1}{\beta+1}}},
	\]
	which is always true.

%

\section{Elimination Algorithm Regret Analysis (Proofs of \Cref{thm:regret-episodes} and \Cref{cor:elim-bounds})}\label{app:sigshift-details}

Following the notation of \Cref{app:base}, we let $c_0,c_1,c_2,\ldots$ denote constants not depending on $T$ , $\kappa_1$, or $\kappa_2$ (\Cref{ass:reservoir}).

\subsection{Details for the Proof of \Cref{thm:regret-episodes}}

As in the proof of \Cref{thm:blackbox} in \Cref{app:blackbox}, we assume WLOG that $\log(T) \geq \kappa_1^{- (1 \land 2\beta)}$ or else we can bound the regret by a constant only depending on $\kappa_1$.

Following the proof outline of \Cref{subsec:regret-upper-bound-elim}, we have
%
	by Freedman's inequality (\Cref{lem:freedman}) with probability at least $1-1/T$:
	\begin{align*}
		\sum_{t = t_{\ell}^m}^{t_{\ell}^{m+1}-1} \delta_t(a_t) &\leq \sum_{t = t_{\ell}^m}^{t_{\ell}^{m+1}-1} \sum_{a \in \mc{G}_t} \frac{\delta_t(a)}{|\mc{G}_t|} + c_{8} \cdot \left( \log(T) + \sqrt{ \log(T) \sum_{t=t_{\ell}^m}^{t_{\ell}^{m+1}-1} \mb{E}[ \delta_t(a_t) \mid \mc{H}_{t-1}] } \right) \\
								       &\leq \frac{3}{2} \sum_{t=1}^{t_{\ell}^{m+1}-1} \sum_{a \in \mc{G}_t} \frac{\delta_t(a)}{|\mc{G}_t|} + c_{9}\log(T),
	\end{align*}
	where the second inequality uses AM-GM.
	Now, recalling that $t^a$ is the last round in block $\ho{t_{\ell}^m}{t_{\ell}^{m+1}}$ that arm $a \in [K]$ is retained, we have: 
	\begin{align}\label{eq:max-Gs}
		\sum_{t=t_{\ell}^m}^{t^a} \frac{\delta_t(a)}{|\mc{G}_t|} &\leq \left( \max_{s \in [t_{\ell}^m,t^a]} \frac{1}{|\mc{G}_s|} \right) \sum_{t=t_{\ell}^m}^{t^a} \delta_t(a).
	\end{align}
	We next again use Freedman's inequality (\Cref{lem:freedman}) to relate $\sum_{t\in I} \delta_t(a)$ to $\sum_{t\in I} \IW_t(a)$. For the variance bound, we have
    \begin{align*}
        \mathbb{E}[(\delta_t(a)-\IW_t(a))^2|\mc{H}_{t-1}]&\le \mathbb{E}[(\IW_t)^2(a)|\mc{H}_{t-1}]\cr &= \mb{E}[ |\mc{G}_t|^2 \cdot (1 - Y_t(a_t))^2 \cdot \pmb{1}\{a_t=a\} \mid \mc{H}_{t-1}] \\
								  &\leq |\mc{G}_t|^2 \cdot \mb{E}[ (1-Y_t(a)) \mid \mc{H}_{t-1}] \cdot \mb{E}[ \pmb{1}\{a_t=a\} \mid \mc{H}_{t-1}]\\
								  &= |\mc{G}_t|^2 \cdot \delta_t(a) \cdot \frac{1}{|\mc{G}_t|} \\
								  &= \delta_t(a) \cdot |\mc{G}_t|.
    \end{align*}
We also have $\max_{s\in I}|\delta_s(a)-\IW_s(a)|\le \max_{s\in I}|\mc{G}_s|$
and $\mb{E}[ \IW_t(a) \mid \mc{H}_{t-1}] = \delta_t(a)$.

	From the above and Freedman's inequality, we can show that w.p. at least $1-1/T$, for any interval $I\subseteq [T]$ on which arm $a$ is retained: 
    \begin{align}
	    \left|\sum_{t = t_{\ell}^m}^{t^a} \delta_t(a) - \IW_t(a) \right| &\le c_{10}\left(  \max_{t \in [t_{\ell}^m,t^a]} |\mc{G}_t| \log(T) + \sqrt{ \log(T) \sum_{t = t_{\ell}^m}^{t^a} \delta_t(a) \cdot |\mc{G}_t|}\right)  \nonumber\\
	    &\le  \half \sum_{t = t_{\ell}^m}^{t^a} \delta_t(a) + c_{11} \max_{t \in [t_{\ell}^m,t^a]} |\mc{G}_t| \log(T), \numberthis\label{eq:delta_error_bd}
    \end{align}
	where again we use AM-GM in the second inequality.
	Moving the $\half \sum_{t} \delta_t(a)$ to the other side, we get
	\[
		\sum_{t = t_{\ell}^m}^{t^a} \delta_t(a) \leq 2 \sum_{t=t_{\ell}^m}^{t^a} \IW_t(a) + c_{12} \log(T) \max_{t \in [t_{\ell}^m,t^a]} |\mc{G}_t|.
	\]
	Plugging the above into \Cref{eq:max-Gs}, we have
	\begin{align*}
		\sum_{t=t_{\ell}^m}^{t^a} \frac{\delta_t(a)}{|\mc{G}_t|} &\leq  \left( \max_{s\in [t_{\ell}^m,t^a]} \frac{c_{13}}{|\mc{G}_s|} \right) \left( \sum_{t=t_{\ell}^m}^{t^a} \IW_t(a) + \log(T) \max_{t \in [t_{\ell}^m,t^a]} |\mc{G}_t| \right)\cr &\leq c_{14}  \max_{s\in [t_{\ell}^m,t^a]} \frac{|\mc{A}_m|}{|\mc{G}_s|}  \log(T) ,
	\end{align*}
where the second inequality is from the elimination guarantee (\Cref{line:elim} of \Cref{alg:elim})  and  $\max_{t \in [t_{\ell}^m,t^a]} |\mc{G}_t| = |\mc{A}_m|$.

    We next have
	\begin{align*}
		\sum_{a \in \mc{A}_m}  \max_{s \in [t_{\ell}^m,t^a-1]} \frac{1}{|\mc{G}_s|} &= \sum_{a \in \mc{A}_m} \frac{1}{|\mc{G}_{t^a-1}|} \cdot \pmb{1}\{ t^a < t_{\ell}^{m+1} - 1\} + \frac{1}{|\mc{G}_{t_{\ell}^{m+1}-1}|} \cdot \pmb{1}\{ t^a = t_{\ell}^{m+1} - 1\} \\
											    &\leq \sum_{i=1}^{|\mc{A}_m|} \frac{1}{i} + \frac{|\mc{G}_{t_{\ell}^{m+1}-1}|}{|\mc{G}_{t_{\ell}^{m+1}-1}|} \\
											    &\leq 1 + \log(|\mc{A}_m|),
	\end{align*}
	Now, combining the above with our previous bound and summing over arms $a$, we have:
	\[
		\sum_{a \in \mc{A}_m} \sum_{t=t_{\ell}^m}^{t_{\ell}^{m+1}-1} \frac{\delta_t(a)}{|\mc{G}_t|} \cdot \pmb{1}\{ a \in \mc{G}_t\} \leq c_{15} |\mc{A}_m| \cdot \log^2(T).
	\]
	We know that for any $m\neq m_\ell$, we have $|\mc{A}_m|=\ceil{ 2^{(m+1) \cdot \frac{\beta}{\beta+1}} \cdot \log(T) }$.
	We also have
	\[
		|\mc{A}_{m_\ell}|\le 2|\mc{A}_{m_\ell-1}|= 2\ceil{ (2 (t_{\ell}^{m_\ell+1} - t_{\ell}^{m_\ell}) )^{\frac{\beta}{\beta+1}} \log(T)  }.
	\]
	Then, summing $|\mc{A}_m| \log^2(T)$ over $m\in[m_\ell]$ and $\ell\in [\hat{L}]$ gives us the regret bound of order $\sum_{\ell=1}^{\hat{L}} (t_{\ell+1} - t_{\ell})^{\frac{\beta}{\beta+1}} \log^3(T)$.

\subsection{Details for the Proof of \Cref{cor:elim-bounds}}

We define event $\mc{E}_5$ such that, for any $I\subseteq [T]$,
	\begin{equation}\label{eq:concentration-cor}
		\left|\sum_{t\in I} \delta_t(a) - \IW_t(a)\right| \le   \half \sum_{t\in I} \delta_t(a) + c_{16} {\log(T)}\max_{t\in I} |\mc{G}_t|,
	\end{equation}
	which holds with probability at least $1-1/T$, by the same reasoning as \eqref{eq:delta_error_bd}
We consider episode $\ell\in[\hat{L}]$ that terminates in a restart and let $m_\ell$ be the index of the last block in $\ell$-th episode and $t_\ell$ be the start time of the $\ell$-th episode.

\paragraph*{$\bullet$ Bounding the Number of Episodes.}
We claim the number of episodes is at most the number of significant shifts or $\hat{L}\le \tilde{L}+1$.
In particular, we prove that a significant shift must have occurred over some block in each episode concluding with a restart.
Let $t^a$ be the time when arm $a$ is eliminated.
For the $m_\ell$-th block $[t_\ell^{m_\ell}, t_{\ell+1})$, on $\mc{E}_5$, for each arm $a\in \mc{A}_{m_\ell}$ we have by \Cref{eq:concentration-cor} and \Cref{line:elim} of \Cref{alg:elim} with $C_2 > 0$ large enough::
\begin{align*}
	\sum_{s=t_\ell^{m_\ell}}^{t^a}\delta_t(a)\ge \frac{2}{3}\left( \sum_{s=t_\ell^{m_\ell}}^{t^a} \IW(a) - c_{18}\log(T)\max_{t\in [t_\ell^{m_\ell},t^a]}|\mc{G}_t|\right) \ge c_{17} |\mc{A}_{m_\ell}|\log(T).
\end{align*}
Then, for all $a\in \mc{A}_{m_\ell}$ where $|\mc{A}_{m_\ell}|=\ceil{2^{(m+1) \cdot \frac{\beta}{\beta+1}} \log(T) } \ge (t_{\ell+1}-t_\ell^{m_\ell})^{\frac{\beta}{\beta+1}} \log(T) \ge (t^a-t_\ell^{m_\ell})^{\frac{\beta}{\beta+1}} \log(T)$, we have
\begin{align}
	\sum_{s=t_\ell^{m_\ell}}^{t^a}\delta_t(a)> 3 (t_{\ell+1}-t_\ell^{m_\ell})^{\frac{\beta}{\beta+1}} \log^3(T) \ge 3 (t^a-t_\ell^{m_\ell})^{\frac{\beta}{\beta+1}} \log^3(T), \label{eq:gap-sum-lbd}
\end{align}
which implies that \Cref{eq:sig-regret} is violated for $\log(T) \geq \kappa_1^{-1}$,
meaning arm $a$ is unsafe.
By \Cref{defn:sig-shift}, a significant shift must have occurred within the block $\ho{t_{\ell}^{m_{\ell}}}{t_{\ell+1}}$.

Now, by considering that the $\hat{L}$-th episode may end by reaching the horizon $T$ rather than restarting, which does not ensure a significant shift, we conclude that $\hat{L}\le \tilde{L}+1$. Therefore, by using Jensen's inequality, we have w.p. at least $1 - T^{-1}$,
 \begin{align}
	 \sum_{\ell=1}^{\hat{L}} (t_{\ell+1}-t_\ell)^{\frac{\beta}{\beta+1}} \le \hat{L}^{\frac{1}{\beta+1}} T^{\frac{\beta}{\beta+1}} \le  (\tilde{L}+1)^{\frac{1}{\beta+1}} T^{\frac{\beta}{\beta+1}}.\label{eq:regret-bd-L}
 \end{align}
\paragraph*{$\bullet$ Bounding the Per-Episode Variation.}
Next, we show that, on event $\mc{E}_5$ where \Cref{eq:concentration-cor} holds, the total rotting variation over episode $\ho{t_{\ell}}{t_{\ell+1}}$ is at least
	\[
		\sum_{t=t_{\ell}}^{t_{\ell+1}-1} (\mu_{t-1}(a_{t-1}) - \mu_{t}(a_{t-1}))_+ \geq \frac{\log^3(T)}{ ( t_{\ell+1} - t_{\ell})^{\frac{1}{\beta+1}}}.
	\]
From \eqref{eq:gap-sum-lbd}, we have for all $a\in \mc{A}_{m_\ell}$, there exists $t'\in[t_{\ell}^{m_\ell},t^a]$ such that
\begin{align}
	\delta_{t'}(a)\ge \frac{3\log^3(T)}{ (t^a-t_\ell^{m_\ell})^{\frac{1}{\beta+1}} }\ge \frac{3\log^3(T)}{ (t_{\ell+1}-t_\ell^{m_\ell})^{\frac{1}{\beta+1}} }. \label{eq:gap-lbd}
\end{align}

Recall $\mu_0(a)$ is the initial mean reward of arm $a$.
Let $\hat{a}_{\ell,m}$ be the best arm in terms of initial reward value $\mu_0(\cdot)$ among the arms $\mc{A}_m$ subsampled for block $\ho{t_{\ell}^m}{t_{\ell}^{m+1}}$.
We have
	\begin{align*}
		\mb{P}\left(\mu_0(\hat{a}_{\ell,m_\ell}) < 1 - \frac{2\log^3(T)}{2^{m_{\ell} \cdot \frac{1}{\beta+1} } } \right) &=      \mb{P}\left(\max_{a\in \mathcal{A}_{m_\ell}}\mu_0(a) < 1 - \frac{2\log^3(T)}{2^{m_{\ell} \cdot \frac{1}{\beta+1} }} \right)    \\
																 &\leq \left( 1 - \kappa_1 \frac{2\log^{3\beta}(T)}{2^{m_{\ell} \cdot \frac{\beta}{\beta+1} }} \right)^{|\mc{A}_{{m_\ell}}|} \\
																 &\leq \exp\left( - |\mc{A}_{{m_\ell}}| \cdot \frac{2 \kappa_1 \log^{3\beta}(T)}{ 2^{m_{\ell} \cdot \frac{\beta}{\beta+1} } } \right) \\
		&\leq \frac{1}{T^2},
	\end{align*}
	where the last inequality follows from $\log(T) \geq \kappa_1^{-\frac{1}{2\beta}}$ and $|\mc{A}_{m_\ell}| \geq 2^{m_{\ell} \cdot \frac{\beta}{\beta+1}} \cdot \log(T)$.
%
	Then, we define $\mc{E}_6$ to be the corresponding high-probability event
	\[
		\{ \forall \ell \in [\hat{L}]: 1-\mu_0(\hat{a}_{\ell,m_\ell})\le 2\log^3(T) \cdot 2^{-m_\ell \cdot \frac{1}{\beta+1}} \},
	\]
	which holds with at least probability of $1-1/T$.
On $\mc{E}_6$, we have
\begin{align}
	\mu_0(\hat{a}_{\ell,m_\ell}) \geq 1 - \frac{2\log^3(T)}{2^{m_{\ell} \cdot \frac{1}{\beta+1}}} \geq 1 - \frac{2\log^3(T)}{ (t_{\ell+1} - t_{\ell}^{m_\ell})^{\frac{1}{\beta+1}}}.\label{eq:goodness-best-subsmapled-arm-eli}
\end{align}
From \eqref{eq:gap-lbd} and \eqref{eq:goodness-best-subsmapled-arm-eli}, we can observe  that
the mean reward of arm $\hat{a}_{\ell,m_\ell}$ must have moved by amount at least
\[
	\frac{\log^3(T)}{ (t_{\ell+1} - t_{\ell}^{m_\ell})^{\frac{1}{\beta+1}} } \geq \frac{\log^3(T)}{ (t_{\ell+1} - t_{\ell})^{\frac{1}{\beta+1}}},
\]
over the course of the block, implying
	\[
		\mu_0(\hat{a}_{\ell,m_\ell})-\mu_{t'}(\hat{a}_{\ell,m_\ell})\ge \frac{\log^3(T)}{ (t_{\ell+1} - t_{\ell}^{m_\ell})^{\frac{1}{\beta+1} } } \geq \frac{\log^3(T)}{ (t_{\ell+1} - t_{\ell})^{\frac{1}{\beta+1}} }.
	\]
	Since the adversary can only modify the rewards of an arm $a_t$ after the round $t$ it is played, the movement must have occurred on rounds where $\hat{a}_{\ell,m_\ell}$ was chosen by the agent.
	Thus,
	\[
		\sum_{t=t_{\ell}}^{t_{\ell+1}-1} (\mu_{t-1}(a_{t-1}) - \mu_{t}(a_{t-1}))_+ \ge \sum_{t=t_{\ell}}^{t_{\ell+1}-1} (\mu_{t-1}(\hat{a}_{\ell,m_\ell}) - \mu_{t}(\hat{a}_{\ell,m_\ell}))_+ \ge \mu_0(\hat{a}_{\ell,m_\ell})-\mu_{t'}(\hat{a}_{\ell,m_\ell}) \geq \frac{\log^3(T)}{ (t_{\ell+1} - t_{\ell})^{\frac{1}{\beta+1}} }.
	\]
Next, this lower bound on the per-episode variation gives us, in an identical manner \Cref{subsec:proof-sum}, a cumulative regret bound:
   \[
	   \sum_{\ell=1}^{\hat{L}} (t_{\ell+1} - t_{\ell})^{\frac{\beta}{\beta+1}} \log^3(T) \leq \log^3(T) \cdot \left( V_R^{\frac{1}{\beta+2}} T^{\frac{\beta+1}{\beta+2}} + T^{\frac{\beta}{\beta+1}} \right).
\]
\section{Verifying \Cref{assumption:base} for UCB}\label{app:base}

Here, we show the UCB algorithm satisfies \Cref{assumption:base}.
The proof will mostly follow the standard proof for showing the classsical logarithmic regret bound \citep[e.g.][Section 7.1]{lattimore}, with some small modifications to account for the mild corruption (\Cref{defn:mild}).

We first present a variant of the UCB1 Algorithm of \citet{auer2002finite}.

\IncMargin{1em}
\begin{algorithm2e}
	\caption{{Variant of UCB1 (Algoritm 3 of \citet{lattimore})}}
 \label{alg:ucb}
 \bld{Input}: number of arms $K$, horizon $T$. \\
 \For{$t=1,\ldots,T$}{
	 Let $N_{t-1}(a) := \sum_{s=1}^{t-1} \pmb{1}\{a_s=a\}$. \\
	 Let $\UCB_a(t-1) := \begin{cases} \infty & N_{t-1}(a) = 0\\
		 \frac{1}{N_{t-1}(a)} \sum_{s=1}^{t-1} \pmb{1}\{ a_s = a\} \cdot Y_s(a) + \sqrt{\frac{2 \log(T)}{N_{t-1}(a)}} & \text{otherwise}
	 \end{cases}$. \\
	 Play arm $a_t = \argmax_{a \in [K]} \UCB_a(t-1)$.\\
 }
\end{algorithm2e}
\DecMargin{1em}

\begin{thm}\label{thm:ucb}
	\Cref{alg:ucb} satisfies \Cref{assumption:base}. 
\end{thm}

\begin{proof}
Our goal is to establish a high-probability regret bound over $t \leq T$ rounds.
First, using \Cref{defn:mild}, we may write the regret as
\[
	\max_{a \in \mc{A}_0} \sum_{s=1}^t \mu_s(a) - \mu_s(a_s) \leq \max_{a \in \mc{A}_0} \sum_{s=1}^t \mu(a) - \mu(a_s) + 2t \cdot \alpha.
\]
Recall $N_t(a) := \sum_{s=1}^t \pmb{1}\{ a_s=a\}$.
Let $\Delta_a$ denote the gap of arm $a \in \mc{A}_0$.
Then, we can decompose the regret based on whether the gap of each arm is large or small and write:
\begin{equation}\label{eq:ucb-bound}
	\max_{a \in \mc{A}_0} \sum_{s=1}^t \mu(a) -\mu(a_s) \leq \sum_{a \in \mc{A}_0} \Delta_a \cdot N_t(a) \cdot \pmb{1}\left\{ \Delta_a > 4\alpha \right\} + 4 t \cdot \alpha.
\end{equation}
It then remains to bound $N_t(a)$ w.h.p for each $a \in \mc{A}_0$ such that $\Delta_a > \alpha$.
WLOG, suppose arm $a=1$ is the best arm among the arms in $\mc{A}_0$, which we'll index by the set $[ |\mc{A}_0|]$.

We claim with probability at least $1-1/T$, for $t$ such that $N_t(a) > \ceil{\frac{32\log(T)}{\Delta_a^2} }$:
\begin{equation}\label{eq:ucb-claim}
	\UCB_a(t) < \UCB_1(t).
\end{equation}
This will allow us to conclude by bounding $N_t(a) \leq \ceil{\frac{32\log(T)}{\Delta_a^2} }$ in \Cref{eq:ucb-bound}.

Let $\hat{\mu}_t(a) := \sum_{s=1}^t Y_s(a) \cdot \pmb{1}\{a_s = a\}$ and let $\ol{\mu}_t(a) := \sum_{s=1}^t \mu_s(a) \cdot \pmb{1}\{ a_s = a\}$.
Then, by Azuma-Hoeffding inequality and a union bound we have
\[
	\mb{P}\left( \forall s \in [t]: |\hat{\mu}_s(a) - \ol{\mu}_s(a)| \geq \sqrt{ \frac{N_s(a) \cdot \log(2T^2)}{2} } \right) \leq 1/T.
\]
Going forward, suppose the above concentration bound holds.

Now, to show the claim, we first note for $t$ such that $N_t(a) > \ceil{\frac{32\log(T)}{\Delta_a^2} }$ and $a \neq 1$ such that $\Delta_a   > 4\alpha$:
\begin{align*}
	\UCB_a(t) &= \hat{\mu}_t(a) + \sqrt{\frac{ 2\log(T)}{N_t(a)} } \\
		  &\leq \ol{\mu}_t(a) + \sqrt{\frac{2\log(T)}{N_t(a)} } + \sqrt{\frac{\log(2T^2)}{2N_t(a)}} \\
		  &\leq \mu(a) + \sqrt{\frac{2 \log(T) }{N_t(a)} } + \sqrt{\frac{\log(2T^2)}{2N_t(a)}} + \alpha \\
		  &< \mu(a) + \frac{\Delta_a}{4} + \frac{\Delta_a}{4} + \frac{\Delta_a}{4} \\
		  &= \mu(a) + \frac{3\Delta_a}{4}.
\end{align*}
where the third inequality is by \Cref{defn:mild}.
Meanwhile,
\begin{align*}
	\UCB_1(t) &= \hat{\mu}_t(1) + \sqrt{\frac{2\log(T)}{N_t(a)} } \\
		  &\geq \ol{\mu}_t(1) + \sqrt{ \frac{2\log(T)}{N_t(a)} } - \sqrt{ \frac{\log(2T^2)}{2N_t(a)}} \\
		  &> \mu(1) - \alpha \\
		  &\geq \mu(1) - \Delta_a/4.
\end{align*}
Thus, claim \Cref{eq:ucb-claim} is shown.
\end{proof}

\begin{rmk}
	Although not shown here for the sake of redundancy, a very similar regret analysis as \Cref{thm:ucb} shows the Successive Elimination algorithm \citep{even-dar} also satisfies \Cref{assumption:base}.
\end{rmk}



\end{document}